\documentclass{ecai} 

\usepackage{latexsym}
\usepackage{amssymb}
\usepackage{amsmath}
\usepackage{amsthm}
\usepackage{booktabs}
\usepackage{enumitem}
\usepackage{graphicx}
\usepackage{color}

\newtheorem{theorem}{Theorem}

\usepackage{threeparttable}
\usepackage{xr}

\usepackage[T1]{fontenc}
\usepackage[utf8]{inputenc}
\usepackage{mathtools}

\usepackage{amssymb,mathrsfs}
\usepackage{amsthm}
\usepackage{bm}
\usepackage{scalerel}
\usepackage{nicefrac}
\usepackage{microtype} 

\usepackage{url}
\usepackage{colortbl}
\usepackage{booktabs}
\usepackage{multirow}
\usepackage{colortbl,xcolor}
\usepackage{soul}
\usepackage{xparse,xstring}
\usepackage{calc}
\usepackage{etoolbox}

\usepackage{array}
\newcolumntype{L}[1]{>{\raggedright\let\newline\\\arraybackslash\hspace{0pt}}m{#1}}
\newcolumntype{C}[1]{>{\centering\let\newline\\\arraybackslash\hspace{0pt}}m{#1}}
\newcolumntype{R}[1]{>{\raggedleft\let\newline\\\arraybackslash\hspace{0pt}}m{#1}}

\makeatletter
\let\MYcaption\@makecaption
\makeatother
\usepackage[font=footnotesize]{subcaption}
\makeatletter
\let\@makecaption\MYcaption
\makeatother

\usepackage{glossaries}
\makeatletter
\sfcode`\.1006

\let\oldgls\gls
\let\oldglspl\glspl

\newcommand\fussy@ifnextchar[3]{%
	\let\reserved@d=#1%
	\def\reserved@a{#2}%
	\def\reserved@b{#3}%
	\futurelet\@let@token\fussy@ifnch}
\def\fussy@ifnch{%
	\ifx\@let@token\reserved@d
		\let\reserved@c\reserved@a
	\else
		\let\reserved@c\reserved@b
	\fi
	\reserved@c}

\renewcommand{\gls}[1]{%
\oldgls{#1}\fussy@ifnextchar.{\@checkperiod}{\@}}
\renewcommand{\glspl}[1]{%
\oldglspl{#1}\fussy@ifnextchar.{\@checkperiod}{\@}}

\newcommand{\@checkperiod}[1]{%
	\ifnum\sfcode`\.=\spacefactor\else#1\fi
}

\robustify{\gls}
\robustify{\glspl}
\makeatother

\newacronym{wrt}{w.r.t.}{with respect to}
\newacronym{RHS}{R.H.S.}{right-hand side}
\newacronym{LHS}{L.H.S.}{left-hand side}
\newacronym{iid}{i.i.d.}{independent and identically distributed}
\newacronym{SOTA}{SOTA}{state-of-the-art}

\usepackage{float}

\ifx\notloadhyperref\undefined
	\ifx\loadbibentry\undefined
            \PassOptionsToPackage{hypertexnames=false}{hyperref}
	\else
		\usepackage{bibentry}
		\makeatletter\let\saved@bibitem\@bibitem\makeatother
            \PassOptionsToPackage{hypertexnames=false}{hyperref}
		\makeatletter\let\@bibitem\saved@bibitem\makeatother
	\fi
\else
	\ifx\loadbibentry\undefined
		\relax
	\else
		\usepackage{bibentry}
	\fi
\fi

\usepackage{cleveref-forward}

\crefname{equation}{}{}
\Crefname{equation}{}{}
\crefname{claim}{claim}{claims}
\crefname{step}{step}{steps}
\crefname{line}{line}{lines}
\crefname{condition}{condition}{conditions}
\crefname{dmath}{}{}
\crefname{dseries}{}{}
\crefname{dgroup}{}{}
\crefname{page}{page}{pages}

\crefname{Problem}{Problem}{Problems}
\crefformat{Problem}{Problem~#2#1#3}
\crefrangeformat{Problem}{Problems~#3#1#4 to~#5#2#6}

\crefname{Theorem}{Theorem}{Theorems}
\crefname{Corollary}{Corollary}{Corollaries}
\crefname{Proposition}{Proposition}{Propositions}
\crefname{Lemma}{Lemma}{Lemmas}
\crefname{Definition}{Definition}{Definitions}
\crefname{Example}{Example}{Examples}
\crefname{Assumption}{Assumption}{Assumptions}
\crefname{Remark}{Remark}{Remarks}
\crefname{Rem}{Remark}{Remarks}
\crefname{remarks}{Remarks}{Remarks}
\crefname{Appendix}{Appendix}{Appendices}
\crefname{Supplement}{Supplement}{Supplements}
\crefname{Exercise}{Exercise}{Exercises}
\crefname{TheoremA}{Theorem}{Theorems}
\crefname{CorollaryA}{Corollary}{Corollaries}
\crefname{PropositionA}{Proposition}{Propositions}
\crefname{LemmaA}{Lemma}{Lemmas}
\crefname{DefinitionA}{Definition}{Definitions}
\crefname{ExampleA}{Example}{Examples}
\crefname{RemarkA}{Remark}{Remarks}
\crefname{AssumptionA}{Assumption}{Assumptions}
\crefname{ExerciseA}{Exercise}{Exercises}
\crefname{algorithm}{Algorithm}{Algorithms}
\crefname{figure}{Fig.}{Figs.}
\crefname{table}{Table}{Tables}
\crefname{section}{Section}{Sections}
\crefname{subsection}{Section}{Sections}
\crefname{subsubsection}{Section}{Sections}

\usepackage{crossreftools}

\usepackage{algorithm}
\usepackage{algpseudocode}

\ifx\loadbreqn\undefined
	\relax
\else
	\usepackage{breqn}
\fi

\makeatletter
\def\cleartheorem#1{%
    \expandafter\let\csname#1\endcsname\relax
    \expandafter\let\csname c@#1\endcsname\relax
}
\def\clearthms#1{ \@for\tname:=#1\do{\cleartheorem\tname} }
\makeatother

\ifx\renewtheorem\undefined
	\ifx\useTheoremCounter\undefined
		\newtheorem{Theorem}{Theorem}
		\newtheorem{Corollary}{Corollary}
		\newtheorem{Proposition}{Proposition}
		\newtheorem{Lemma}{Lemma}
	\else

	\fi

\fi

\theoremstyle{remark}

\theoremstyle{plain}

\newcommand{\qednew}{\nobreak \ifvmode \relax \else
		\ifdim\lastskip<1.5em \hskip-\lastskip
			\hskip1.5em plus0em minus0.5em \fi \nobreak
		\vrule height0.75em width0.5em depth0.25em\fi}



\newcommand{\nn}{\nonumber\\ }

\NewDocumentCommand{\movedownsub}{e{^_}}{%
	\IfNoValueTF{#1}{%
		\IfNoValueF{#2}{^{}}
	}{%
		^{#1}
	}%
	\IfNoValueF{#2}{_{#2}}
}

\newcommand{\calE}{\mathcal{E}}
\newcommand{\calF}{\mathcal{F}}

\newcommand{\calL}{\mathcal{L}}
\newcommand{\calM}{\mathcal{M}}
\newcommand{\calN}{\mathcal{N}}

\newcommand{\ba}{\mathbf{a}}
\newcommand{\bA}{\mathbf{A}}

\newcommand{\bD}{\mathbf{D}}

\newcommand{\bI}{\mathbf{I}}

\newcommand{\bJ}{\mathbf{J}}

\newcommand{\bv}{\mathbf{v}}
\newcommand{\bV}{\mathbf{V}}

\newcommand{\bW}{\mathbf{W}}
\newcommand{\bx}{\mathbf{x}}
\newcommand{\bX}{\mathbf{X}}

\DeclareSymbolFont{bsfletters}{OT1}{cmss}{bx}{n}
\DeclareSymbolFont{ssfletters}{OT1}{cmss}{m}{n}
\DeclareMathSymbol{\bsfGamma}{0}{bsfletters}{'000}
\DeclareMathSymbol{\ssfGamma}{0}{ssfletters}{'000}
\DeclareMathSymbol{\bsfDelta}{0}{bsfletters}{'001}
\DeclareMathSymbol{\ssfDelta}{0}{ssfletters}{'001}
\DeclareMathSymbol{\bsfTheta}{0}{bsfletters}{'002}
\DeclareMathSymbol{\ssfTheta}{0}{ssfletters}{'002}
\DeclareMathSymbol{\bsfLambda}{0}{bsfletters}{'003}
\DeclareMathSymbol{\ssfLambda}{0}{ssfletters}{'003}
\DeclareMathSymbol{\bsfXi}{0}{bsfletters}{'004}
\DeclareMathSymbol{\ssfXi}{0}{ssfletters}{'004}
\DeclareMathSymbol{\bsfPi}{0}{bsfletters}{'005}
\DeclareMathSymbol{\ssfPi}{0}{ssfletters}{'005}
\DeclareMathSymbol{\bsfSigma}{0}{bsfletters}{'006}
\DeclareMathSymbol{\ssfSigma}{0}{ssfletters}{'006}
\DeclareMathSymbol{\bsfUpsilon}{0}{bsfletters}{'007}
\DeclareMathSymbol{\ssfUpsilon}{0}{ssfletters}{'007}
\DeclareMathSymbol{\bsfPhi}{0}{bsfletters}{'010}
\DeclareMathSymbol{\ssfPhi}{0}{ssfletters}{'010}
\DeclareMathSymbol{\bsfPsi}{0}{bsfletters}{'011}
\DeclareMathSymbol{\ssfPsi}{0}{ssfletters}{'011}
\DeclareMathSymbol{\bsfOmega}{0}{bsfletters}{'012}
\DeclareMathSymbol{\ssfOmega}{0}{ssfletters}{'012}

\newcommand{\btheta}{\bm{\theta}}

\newcommand{\bomega}{\bm{\omega}}

\newcommand{\bOmega}{\bm{\Omega}}

\makeatletter
\newcommand*\rel@kern[1]{\kern#1\dimexpr\macc@kerna}
\newcommand*\widebar[1]{%
  \begingroup
  \def\mathaccent##1##2{%
    \rel@kern{0.8}%
    \overline{\rel@kern{-0.8}\macc@nucleus\rel@kern{0.2}}%
    \rel@kern{-0.2}%
  }%
  \macc@depth\@ne
  \let\math@bgroup\@empty \let\math@egroup\macc@set@skewchar
  \mathsurround\z@ \frozen@everymath{\mathgroup\macc@group\relax}%
  \macc@set@skewchar\relax
  \let\mathaccentV\macc@nested@a
  \macc@nested@a\relax111{#1}%
  \endgroup
}
\makeatother

\DeclareMathOperator*{\argmin}{arg\,min}

\DeclareMathOperator{\var}{var}

\DeclareMathOperator{\cov}{cov}

\newcommand{\ifbcdot}[1]{\ifblank{#1}{\cdot}{#1}}

\DeclarePairedDelimiterX\abs[1]{\lvert}{\rvert}{\ifbcdot{#1}}
\DeclarePairedDelimiterX\parens[1]{(}{)}{\ifbcdot{#1}}
\DeclarePairedDelimiterX\brk[1]{[}{]}{\ifbcdot{#1}}
\DeclarePairedDelimiterX\braces[1]{\{}{\}}{\ifbcdot{#1}}
\DeclarePairedDelimiterX\angles[1]{\langle}{\rangle}{\ifblank{#1}{\cdot,\cdot}{#1}}
\DeclarePairedDelimiterX\ip[2]{\langle}{\rangle}{\ifbcdot{#1},\ifbcdot{#2}}
\DeclarePairedDelimiterX\norm[1]{\lVert}{\rVert}{\ifbcdot{#1}}
\DeclarePairedDelimiterX\ceil[1]{\lceil}{\rceil}{\ifbcdot{#1}}
\DeclarePairedDelimiterX\floor[1]{\lfloor}{\rfloor}{\ifbcdot{#1}}

\DeclareFontFamily{U}{matha}{\hyphenchar\font45}
\DeclareFontShape{U}{matha}{m}{n}{
      <5> <6> <7> <8> <9> <10> gen * matha
      <10.95> matha10 <12> <14.4> <17.28> <20.74> <24.88> matha12
      }{}
\DeclareSymbolFont{matha}{U}{matha}{m}{n}
\DeclareFontSubstitution{U}{matha}{m}{n}

\DeclareFontFamily{U}{mathx}{\hyphenchar\font45}
\DeclareFontShape{U}{mathx}{m}{n}{
      <5> <6> <7> <8> <9> <10>
      <10.95> <12> <14.4> <17.28> <20.74> <24.88>
      mathx10
      }{}
\DeclareSymbolFont{mathx}{U}{mathx}{m}{n}
\DeclareFontSubstitution{U}{mathx}{m}{n}

\DeclareMathDelimiter{\vvvert}{0}{matha}{"7E}{mathx}{"17}
\DeclarePairedDelimiterX\vertiii[1]{\vvvert}{\vvvert}{\ifbcdot{#1}}

\DeclarePairedDelimiterXPP\trace[1]{\operatorname{Tr}}{(}{)}{}{\ifbcdot{#1}} 
\DeclarePairedDelimiterXPP\col[1]{\operatorname{col}}{\{}{\}}{}{\ifbcdot{#1}} 
\DeclarePairedDelimiterXPP\row[1]{\operatorname{row}}{\{}{\}}{}{\ifbcdot{#1}} 
\DeclarePairedDelimiterXPP\erf[1]{\operatorname{erf}}{(}{)}{}{\ifbcdot{#1}}
\DeclarePairedDelimiterXPP\erfc[1]{\operatorname{erfc}}{(}{)}{}{\ifbcdot{#1}}
\DeclarePairedDelimiterXPP\KLD[2]{D}{(}{)}{}{\ifbcdot{#1}\, \delimsize\|\, \ifbcdot{#2}} 
\DeclarePairedDelimiterXPP\op[2]{\operatorname{#1}}{(}{)}{}{#2} 

\newcommand{\T}{^{\mkern-1.5mu\mathop\intercal}}

\DeclarePairedDelimiterXPP\indicate[1]{{\bf 1}}{\{}{\}}{}{\ifbcdot{#1}}

\newcommand{\tc}[1]{^{(#1)}}
\NewDocumentCommand\ofrac{s m}{%
	\IfBooleanTF#1%
	{\dfrac{1}{#2}}%
	{\frac{1}{#2}}%
}
\NewDocumentCommand\ddfrac{s m m}{%
	\IfBooleanTF#1%
	{\dfrac{\mathrm{d} {#2}}{\mathrm{d} {#3}}}%
	{\frac{\mathrm{d} {#2}}{\mathrm{d} {#3}}}%
}
\NewDocumentCommand\ppfrac{s m m}{%
	\IfBooleanTF#1%
	{\dfrac{\partial {#2}}{\partial {#3}}}%
	{\frac{\partial {#2}}{\partial {#3}}}%
}

\newcommand{\setgiven}{:}
\providecommand\given{}

\DeclarePairedDelimiterX\Set[2]\{\}{%
	\if#1:%
		\renewcommand\given{\SetSymbol{:}}%
	\else%
		\renewcommand\given{\SetSymbol[\delimsize]{#1}}%
	\fi%
#2
}

\NewDocumentCommand\set{s O{\setgiven} m}{%
	\IfBooleanTF#1%
	{\Set*{#2}{#3}}%
	{\Set{#2}{#3}}%
}

\providecommand\given{}
\DeclarePairedDelimiterXPP\cprob[1]{}(){}{
\renewcommand\given{\nonscript\,\delimsize\vert\allowbreak\nonscript\,\mathopen{}}%
#1%
}
\DeclarePairedDelimiterXPP\cexp[1]{}[]{}{
\renewcommand\given{\nonscript\,\delimsize\vert\allowbreak\nonscript\,\mathopen{}}%
#1%
}

\DeclareDocumentCommand \P { s e{_^} d() g } {%
	\mathbb{P}%
	\IfBooleanTF{#1}%
		{
			\IfValueT{#2}{_{#2}}%
			\IfValueT{#3}{^{#3}}%
			\IfValueTF{#5}{\cprob{#4 \given #5}}{\IfValueT{#4}{\cprob{#4}}}%
		}%
		{
			\IfValueT{#2}{_{#2}}%
			\IfValueT{#3}{^{#3}}%
			\IfValueTF{#5}{\cprob*{#4 \given #5}}{\IfValueT{#4}{\cprob*{#4}}}%
		}%
}

\DeclareDocumentCommand \E { s e{_^} o g } {%
	\mathbb{E}%
	\IfBooleanTF{#1}%
		{
			\IfValueT{#2}{_{#2}}%
			\IfValueT{#3}{^{#3}}%
			\IfValueTF{#5}{\cexp{#4 \given #5}}{\IfValueT{#4}{\cexp{#4}}}%
		}%
		{
			\IfValueT{#2}{_{#2}}%
			\IfValueT{#3}{^{#3}}%
			\IfValueTF{#5}{\cexp*{#4 \given #5}}{\IfValueT{#4}{\cexp*{#4}}}%
		}%
}

\DeclareDocumentCommand \Var { s e{_^} d() g } {%
	\var%
	\IfBooleanTF{#1}%
		{
			\IfValueT{#2}{_{#2}}%
			\IfValueT{#3}{^{#3}}%
			\IfValueTF{#5}{\cprob{#4 \given #5}}{\IfValueT{#4}{\cprob{#4}}}%
		}%
		{
			\IfValueT{#2}{_{#2}}%
			\IfValueT{#3}{^{#3}}%
			\IfValueTF{#5}{\cprob*{#4 \given #5}}{\IfValueT{#4}{\cprob*{#4}}}%
		}%
}

\DeclareDocumentCommand \Cov { s e{_^} d() g } {%
	\cov%
	\IfBooleanTF{#1}%
		{
			\IfValueT{#2}{_{#2}}%
			\IfValueT{#3}{^{#3}}%
			\IfValueTF{#5}{\cprob{#4 \given #5}}{\IfValueT{#4}{\cprob{#4}}}%
		}%
		{
			\IfValueT{#2}{_{#2}}%
			\IfValueT{#3}{^{#3}}%
			\IfValueTF{#5}{\cprob*{#4 \given #5}}{\IfValueT{#4}{\cprob*{#4}}}%
		}%
}

\NewDocumentCommand {\cbrace} {t+ D[]{black} D(){\widthof{#5}} m m } {%
	\begingroup%
		\color{#2}
		\IfBooleanTF{#1}{%
			\overbrace{#4}^%
		}{
			\underbrace{#4}_%
		}%
		{\parbox[c]{#3}{\centering\footnotesize{#5}}}%
	\endgroup%
}

\let\oldforall\forall
\renewcommand{\forall}{\oldforall \, }

\let\oldexist\exists
\renewcommand{\exists}{\oldexist \, }

\makeatletter
\newcommand{\udcloser}[1]{\underline{\smash{#1}}}

\newcommand{\rankcolor}[2]{%
	\expandafter\renewcommand\csname #1\endcsname[1]{%
		\ifblank{##1}{%
			{\color{#2} \textbf{#2}}%
		}{%
			\ifmmode
				\textcolor{#2}{\bm{##1}}%
			\else%
				{\color{#2} \textbf{##1}}%
			\fi	
		}%
	}
}

\rankcolor{first}{red}
\rankcolor{second}{blue}
\rankcolor{third}{cyan}
\makeatother

\graphicspath{{./Figures/}{./figures/}}
\DeclareDocumentCommand{\includeCroppedPdf}{ o O{./Figures/} m }{
	\IfFileExists{#2#3-crop.pdf}{}{%
		\immediate\write18{pdfcrop #2#3.pdf #2#3-crop.pdf}}%
	\includegraphics[#1]{#2#3-crop.pdf}
}

\makeatletter
\newcommand*{\addFileDependency}[1]{
  \typeout{(#1)}
  \@addtofilelist{#1}
  \IfFileExists{#1}{}{\typeout{No file #1.}}
}
\makeatother

\definecolor{gray90}{gray}{0.9}
\def\colorlist{red,blue,brown,cyan,darkgray,gray,lightgray,green,lime,magenta,olive,orange,pink,purple,teal,violet,white,yellow}

\makeatletter
\def\startcomment{[}
\ifx\nohighlights\undefined
	\newcommand{\createcolor}[1]{%
			\expandafter\newcommand\csname #1\endcsname[1]{{\color{#1} ##1}}%
	}
	\newcommand{\msout}[1]{\text{\color{green} \st{\ensuremath{#1}}}}
	\newcommand{\del}[1]{{\color{green}\ifmmode \msout{#1}\else\st{#1}\fi}}
\else
	\newcommand{\createcolor}[1]{%
			\expandafter\newcommand\csname #1\endcsname[1]{%
				\noexpandarg%
				\StrChar{##1}{1}[\firstletter]%
				\if\firstletter\startcomment%
					\relax
				\else%
					##1
				\fi
			}%
	}
	\newcommand{\msout}[1]{}
	\newcommand{\del}[1]{}
\fi

\def\@tempa#1,{%
    \ifx\relax#1\relax\else
        \createcolor{#1}%
        \expandafter\@tempa
    \fi
}
\expandafter\@tempa\colorlist,\relax,
\makeatother

\newcommand{\hhide}[1]{}

\ifx\diagnoselabel\undefined
	\relax
\else
	\makeatletter
	\def\@testdef #1#2#3{%
		\def\reserved@a{#3}\expandafter \ifx \csname #1@#2\endcsname
			\reserved@a  \else
			\typeout{^^Jlabel #2 changed:^^J%
				\meaning\reserved@a^^J%
				\expandafter\meaning\csname #1@#2\endcsname^^J}%
			\@tempswatrue \fi}
	\makeatother
\fi

\newcommand{\BibTeX}{B\kern-.05em{\sc i\kern-.025em b}\kern-.08em\TeX}

\newacronym{FL}{FL}{federated learning}
\newacronym{PFL}{PFL}{personalized federated learning}
\newacronym{GNN}{GNN}{graph neural network}
\newacronym{MLP}{MLP}{multi-layer perceptron}

\begin{document}


\begin{frontmatter}


\paperid{9084} 


\title{Personalized Subgraph Federated Learning \\
 with Sheaf Collaboration}


\author[A]{\fnms{Wenfei}~\snm{Liang}\thanks{Corresponding Author. Email: wenfei001@e.ntu.edu.sg.}}
\author[A]{\fnms{Yanan}~\snm{Zhao}}
\author[A]{\fnms{Rui}~\snm{She}} 
\author[A]{\fnms{Yiming}~\snm{Li}}
\author[A]{\fnms{Wee Peng}~\snm{Tay}}

\address[A]{School of Electrical and Electronic Engineering, Nanyang Technological University, Singapore}


\begin{abstract}
Graph-structured data is prevalent in many applications. In subgraph federated learning (FL), this data is distributed across clients, each with a local subgraph. Personalized subgraph FL aims to develop a customized model for each client to handle diverse data distributions. However, performance variation across clients remains a key issue due to the heterogeneity of local subgraphs.
To overcome the challenge, we propose \textit{FedSheafHN}, a novel framework built on a sheaf collaboration mechanism to unify enhanced client descriptors with efficient personalized model generation. Specifically, FedSheafHN embeds each client's local subgraph into a server-constructed collaboration graph by leveraging graph-level embeddings and employing sheaf diffusion within the collaboration graph to enrich client representations. Subsequently, FedSheafHN generates customized client models via a server-optimized hypernetwork. Empirical evaluations demonstrate that FedSheafHN outperforms existing personalized subgraph FL methods on various graph datasets. Additionally, it exhibits fast model convergence and effectively generalizes to new clients.
\end{abstract}
\end{frontmatter}
\section{Introduction}
Many \glspl{GNN} focus on a single graph \cite{she2023image,zhao2024distributed}, storing nodes and edges from various sources in a central server. Collaborative training of GNNs across distributed graphs can be achieved through \gls{FL} \cite{zhang2021subgraph,pillutla2022federated}, where each participant trains a local \gls{GNN}, and a central server aggregates their updated weights. In practice, clients may have unique subgraphs, and local data distributions can vary significantly. Some clients may even tackle distinct tasks. \Gls{PFL} \cite{smith2017federated} addresses this by allowing each client to use a personalized rather than a shared model.

This paper focuses on subgraph FL, a particularly challenging setting where \textbf{clients hold largely disjoint subsets of a global graph}.  
A key challenge is balancing the benefits of joint training with maintaining unique models tailored to the local data and task of each client. 
Recent subgraph FL methods \cite{zhang2021subgraph,baek2023personalized} attempt to address this by either addressing missing edges across subgraphs, e.g., by expanding subgraphs using information from others, or by identifying community structures to guide model aggregation.
However, as shown in Fig.~\ref{fig_client_accs}(a), \textbf{existing PFL methods suffer from significant performance variation across clients}, largely due to the heterogeneity of subgraphs from different parts of a global graph. 
We ask the following two questions:
\begin{enumerate}[align=left,label=\textbf{Q\arabic*}]
    \item\label{Q1} \textbf{How can we learn and exploit the underlying relationships among clients to support effective joint training?}
    \item\label{Q2} \textbf{How can we maintain personalized models when clients have distinct subgraphs and data distributions?}
\end{enumerate}

\begin{figure}[H]
\centering
\includegraphics[width=\columnwidth]{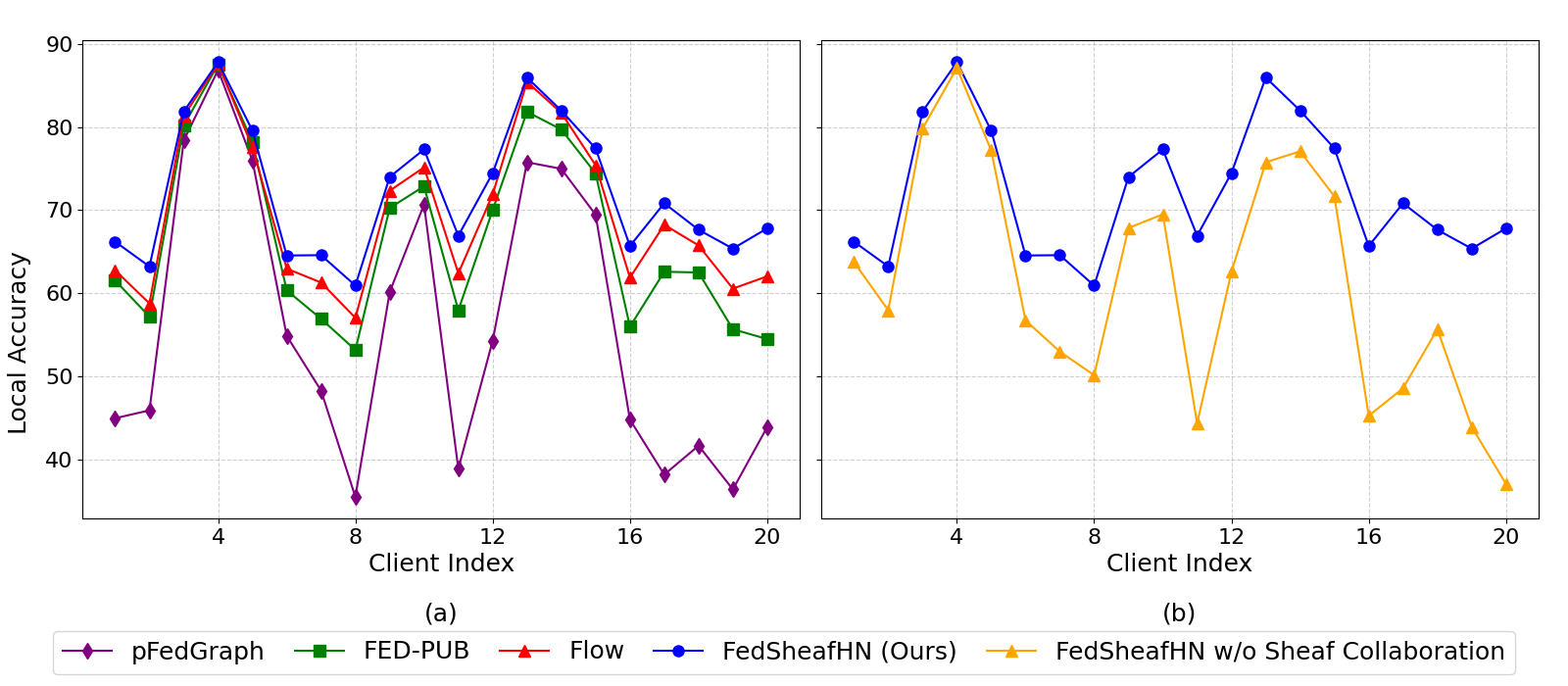}
\caption{Local accuracy comparison across 20 clients in non-overlapping scenario on the ogbn-arxiv dataset. (a) Comparison with PFL baselines: Flow, FED-PUB, and pFedGraph, demonstrating more stable and balanced performance among all clients for FedSheafHN. (b) Comparison shows FedSheafHN with sheaf collaboration (71.88$\pm$7.92) outperforms the variant without it (61.26$\pm$13.77), highlighting the effectiveness of sheaf collaboration. }
\label{fig_client_accs}
\end{figure}

To address \ref{Q1}, we propose sheaf collaboration, which enables effective cross-client collaboration by leveraging a server-side collaboration graph as a manifold to model underlying client relationships. This facilitates both information exchange and progressive refinement of client descriptors. 
To address \ref{Q2}, we employ Hypernetworks (HNs) to generate client-specific model parameters from descriptive vectors, enabling flexibility and diversity in personalized client models. The HNs map input embedding to model parameters, effectively constraining all personalized models to lie on a shared manifold in the parameter space \cite{shamsian2021personalized}. This ensures each client's model is adapted to its data while still maintaining a structured space for generalization.

We call our model \textit{Federated learning with Sheaf collaboration and HyperNetworks} (FedSheafHN). 
FedSheafHN facilitates in-depth cross-client information exchange and generates highly personalized models from enriched client representations. Rather than relying on single-scalar similarity or partial parameter sharing, sheaf collaboration leverages the underlying structure of the client relationship and captures rich, client-specific mappings that align local distributions without sacrificing individuality, as demonstrated in Fig.~\ref{fig_client_accs}(b).  
By equipping these enriched representations with HNs, FedSheafHN provides a promising approach to the core personalization challenges while producing models better aligned to heterogeneous demands, ultimately offering more stable and balanced performance for all clients.
Our main contributions are summarized as follows:
\begin{itemize}
\item We propose sheaf collaboration framework that leverages sheaf diffusion with server-constructed collaboration graph to capture the intricate inter-client relationships, thereby enabling more effective information aggregation across heterogeneous clients.

\item We develop an attention-based HNs and optimize it batch-wise on the server, efficiently generating personalized model parameters using enriched client descriptions produced by sheaf collaboration.

\item We provide a theoretical analysis of FedSheafHN, establishing both convergence and generalization guarantees.

\item We conduct extensive experiments on multiple graph-structured datasets under heterogeneous subgraph FL scenarios. FedSheafHN consistently outperforms state-of-the-art baselines in most cases, while converging quickly and generalizing effectively to newly joined clients.
\end{itemize}
\emph{Notations}: For a positive integer $N$, we let $[N] = \set{1, 2, \dots N}$. A matrix $\bA$ is written as $\bA = (\ba_i)_{i\in[N]}$ if its $i$-th row is the row vector $\ba_i$, for $i=1,2,\dots,N$. 
A vectorized version of $\bA$ is denoted as $\tilde{\bA}$, which is formed by stacking the transposes of the row vectors $\ba_i$ for $i = 1, 2, \dots, N$ into a single vector.
$(\cdot)\T$ denotes the transpose operator.
A graph $G$ is denoted as $G=(\calN,\calE)$ where $\calN$ is the vertex set and $\calE$ is the edge set, or $G=(\calN,\calE,\calF)$ if it is associated with a feature matrix $\calF$. $\bigoplus$ denotes the direct sum of vector spaces.
\section{Related work}\label{sect_relatedworks}
\subsection{Federated Learning}
\paragraph{PFL}
To address challenges related to data and device heterogeneity in FL, various PFL methods have been introduced, encompassing various approaches such as local fine-tuning \cite{arivazhagan2019federated,schneider2021personalization}, regularization for objective functions \cite{hanzely2020lower,yan2024personalized}, model mixing \cite{ma2022layer}, meta-learning \cite{jiang2019improving,lee2024fedl2p}, personalized parameter decomposition \cite{arivazhagan2019federated,panchal2024flow}, and differentially privacy \cite{agarwal2021skellam,li2024clients}.
Additionally, the concept of training multiple global models at the server has been explored for efficient PFL \cite{ghosh2020efficient,huang2021personalized}. This method involves training different models for distinct client groups, and clustering clients based on similarity.
Another strategy is to train individual client models collaboratively \cite{huang2021personalized,zhang2021parameterized,ye2023personalized,wang2023fedgs,scott2024pefll}.

\paragraph{Graph FL}
Recent research highlights the potential of integrating the FL framework into collaborative GNNs training, focusing on subgraph- and graph-level methods. Graph-level FL assumes that clients have disjoint graphs, such as molecular graphs. 
Studies \cite{xie2021federated,he2021spreadgnn,tan2023federated} delved into addressing heterogeneity in non-IID graphs with varying labels.
In contrast to graph-level FL, subgraph-level FL introduces unique challenges related to graph structures, such as missing links between subgraphs of a global graph. Methods like \cite{wu2021fedgnn,zhang2021subgraph} addressed this by augmenting nodes and connecting them, while \cite{baek2023personalized} explored subgraph communities as densely connected clusters of subgraphs to address the problem.
Unlike the aforementioned subgraph FL methods, our approach facilitates sheaf collaboration at the server to capture the underlying geometry and client interconnections, enabling effective representation for heterogeneous clients. This process can be considered as node-client two-level graph learning, improving the model's capacity to capture and utilize diverse information.

\subsection{Hypernetworks}
HNs \cite{klein2015dynamic} are deep neural networks that generate weights for a target network. The output weights are dynamically adjusted based on the input to the HNs \cite{klocek2019hypernetwork,navon2020learning}.
SMASH \cite{brock2017smash} extended HNs to Neural Architecture Search (NAS) by encoding architectures as 3D tensors using a memory channel scheme. In contrast, our approach encodes networks as computation graphs and employs GNNs.
GHN \cite{zhang2018graph} was introduced for anytime prediction tasks, optimizing both speed and the speed-accuracy trade-off curve. pFedHN \cite{shamsian2021personalized} leveraged HNs in FL with input task embeddings, while hFedGHN \cite{xu2023heterogeneous} extended GHN to train heterogeneous local models based on graph learning.
HNs inherently lend themselves to learn a diverse set of personalized models, and our approach equips them with an attention layer to further aggregate cross-client information, optimizing the HNs batch-wise to efficiently generate highly personalized model parameters from enriched client descriptions.

\section{Methodology} \label{sect_model}

\begin{figure*}[!t] 
\centering
\includegraphics[width=\textwidth]{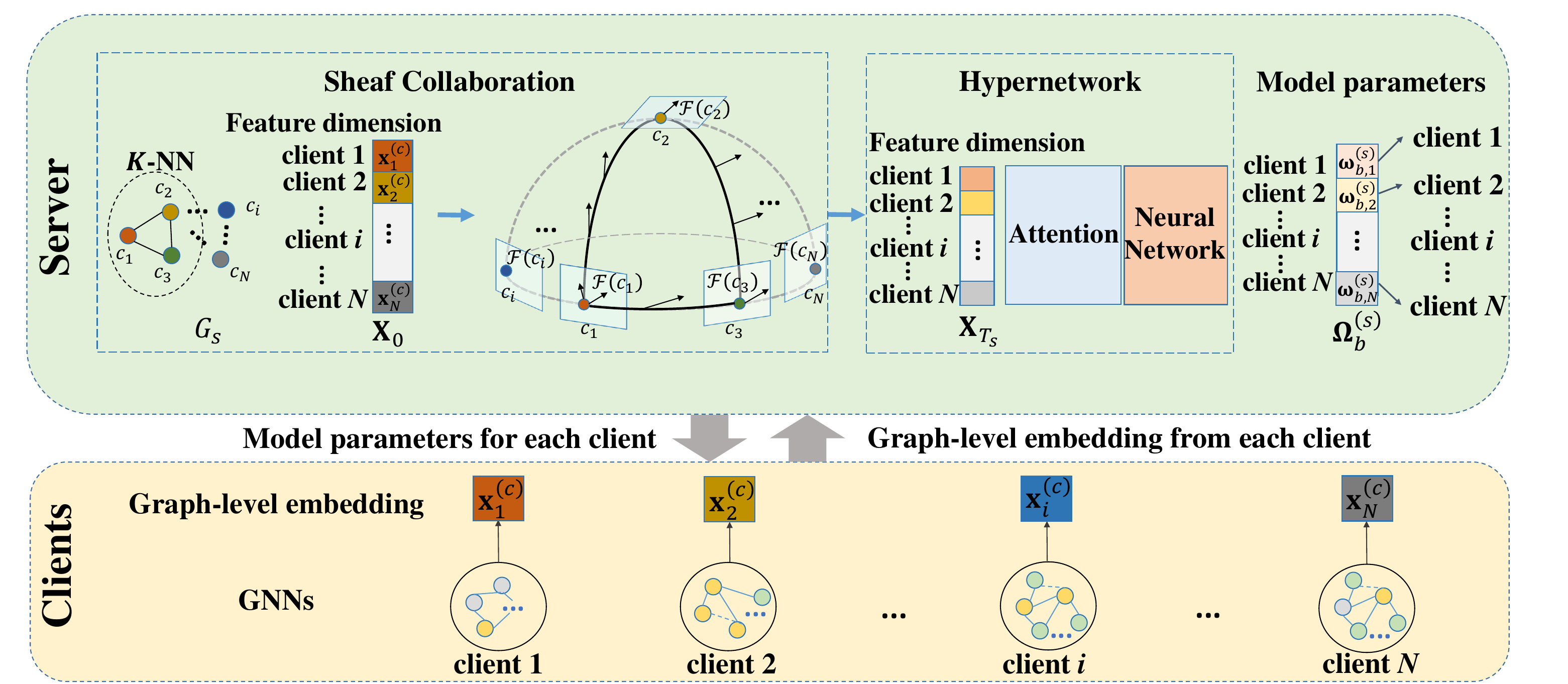} 
\caption{Overview of the FedSheafHN framework. Clients train GNNs to process local graph-structured data and generate graph-level embeddings. The server constructs a client collaboration graph from these embeddings to capture inter-client relationships. It updates the embeddings using sheaf diffusion, and then employs an attention-based hypernetwork to generate personalized model parameters based on the enriched client descriptions. }
\label{fig_wholeframework}
\vspace{2mm}
\end{figure*}

\subsection{The FedSheafHN Framework}
We consider a FL scenario with $N$ clients and a central server. Every client $i \in [N]$ possesses a local graph $G_{i}=(\calM_i, \calE_i, \bV_i)$, where $\calM_i$ is the node set, $\calE_i$ is the edge set, and $\bV_i$ is a feature matrix. Denote $\bV_{i} = (\bv_{i,m})_{m\in \calM_i}$, whose rows are the node feature vectors $\bv_{i,m}$ for each node $m\in \calM_{i}$.
Each client aims to learn a local model, denoted as $f_i(G_{i};\bomega_{i})$, tailored to its specific task. Here, $\bomega_{i}$ represents the model parameters for the client $i$. It is important to note that each local graph $G_{i}$ can potentially be a subgraph of a larger, undisclosed global graph. 

FedSheafHN aims to enhance the performance of the personalized model $f_i(G_{i};\bomega_{i})$ on each client through subgraph FL. We refer the reader to Fig.~\ref{fig_wholeframework} for an overview of the FedSheafHN pipeline. FedSheafHN constructs a collaboration graph $G_s$ on the server using the client graph-level embeddings transmitted from clients. The server generates personalized model parameters $\bomega_{b,i}\tc{s}$ for each client $i$ using the sheaf diffusion model $S(\cdot; \btheta)$ and hypernetwork $H(\cdot; \bm{\varphi})$. For ease of reference, we call $\bomega_{b,i}\tc{s}$ the \emph{backbone} parameter for client $i$. The server then distributes backbone parameters $(\bomega_{b,i}\tc{s})_{i\in [N]}$ to clients.

The model parameters $\bomega_{i} =(\bomega_{b,i}, \bomega_{h,i})$ of each client $i$ are composed of two parts: 
the backbone parameter $\bomega_{b,i}$ (initially trained locally and then updated by the server) and the head parameters $\bomega_{h,i}$ (trained locally). 
Each client $i$ initializes $\bomega_{b,i}=\bomega_{b,i}\tc{s}$ and $\bomega_{h,i}$ as the previously trained head parameters (with random initialization before the first FL round).
It then performs $T_c$ epochs of local training on local graph $G_{i}$ to update the model parameters $\bomega_{i}$. Let the model parameters after $T_c$ epochs of local training be $\bomega_{i}\tc{c}=(\bomega_{b,i}\tc{c},\bomega_{h,i}\tc{c})$. 
Each client $i$ transmits its incremental difference in its backbone parameter $\Delta \bomega_{b,i}:= \bomega_{b,i}\tc{c}-\bomega_{b,i}\tc{s}$ to the server. 
The server updates the sheaf diffusion model $S$ and hypernetwork $H$. This process is repeated for multiple communication rounds.

The objective of each client $i$ is to find  
\begin{align}\label{loss_funci}
\argmin_{\bomega_{i}} \mathcal{L}_{i}(\bomega_{i}),
\end{align}
where $\mathcal{L}_{i}(\cdot)$ denotes the task-specific loss of client $i$. The global training objective at the server is defined as
\begin{align}
&\argmin_{(\bomega_{b,i})_{i \in [N]}} \mathcal{L}((\bomega_{b,i})_{i \in [N]})
=\argmin_{(\bomega_{b,i})_{i \in [N]}}  \sum_{i=1}^{N} \mathcal{L}_{i} (\bomega_{b,i},\bomega_{h,i}\tc{c}) \nn
&\qquad\qquad= \argmin_{\btheta,\bm{\varphi}}  \sum_{i=1}^{N} \mathcal{L}_{i} (\btheta ,\bm{\varphi},\bomega_{h,i}\tc{c}), \label{loss_func}
\end{align}
where the last equality holds because $\bomega_{b,i}$ is function of sheaf diffusion parameters $\btheta$ and hypernetwork parameters $\bm{\varphi}$.\footnote{We abuse notations in Eq.~(\ref{loss_func}) by writing the same symbol $\calL_i$.}

In the following subsections, we provide detailed descriptions of each component of the FedSheafHN framework.

\subsection{Sheaf Collaboration}
\paragraph{Collaboration Graph Construction}
In subgraph FL, the server faces challenges in perceiving the overarching graph structure, complicating effective client collaboration. To address this, we formulate a novel collaboration graph leveraging client data-driven features to capture intrinsic geometric relationships. This graph facilitates client collaboration and enhances client-specific data representation.

Let $\bv_{i,m}^{(c)}$ be the updated embedding vector of node $m\in \calM_i$ of client $i$ after $T_c$ epochs local training of the client model $f_i(G_i;\bomega_{i})$. The graph-level embedding of client $i$ is computed locally as 
\begin{align}
\label{graph_embedding}
\bx_{i}\tc{c}=\frac{1}{\left | \calM_{i} \right |} \sum_{m\in \calM_{i}}  \bv_{i,m}^{(c)},
\end{align}
and transmitted to the server. It serves as the feature representation of client $i$. After receiving the graph-level embeddings from all clients, the server constructs a collaboration graph 
\begin{align}\label{Gs}
G_s=([N],\calE_s,\bX_0),
\end{align}
where $\calE_s$ is the edge set, and $\bX_0=(\bx_i\tc{c})_{i\in [N]}$ is the feature matrix consisting of graph-level embeddings of all clients stacked in row form. 
The edge set is formed using a $K$-nearest neighbors (KNN) approach based on a predefined distance (e.g., cosine similarity) between the graph-level embeddings.
To conserve communication resources, the client collaboration graph $G_s$ on the server is updated every $r_{in}$ communication rounds. Specifically, after $r_{in}$ communication rounds, we reset $\bX_0$ using the graph-level embeddings transmitted by the clients, i.e., $(\bx_i\tc{c})_{i \in [N]}$.

\paragraph{Graph Representation Learning} 
To uncover latent inter-client relationships and enrich client descriptions through information aggregation, our method incorporates elements of (cellular) sheaf theory. A cellular sheaf assigns a vector space to each node and edge in a graph, with linear maps for incident node-edge pairs \cite{curry2014sheaves,hansen2020sheaf,bodnar2022neural}. These vector spaces represent points on a manifold, with the sheaf Laplacian modeling discrete transport of elements via rotations in neighboring vector spaces, where the graph acts as the manifold. Inspired by this, our framework employs the server-side collaboration graph as a manifold to model client relationships.

Mathematically, a cellular sheaf $(G_s,\mathcal{F})$ on the collaboration graph $G_s$ includes three components:
\textit{(a)} a vector space $\mathcal{F}(n)$ for each $n \in [N]$ containing the graph-level embedding vector of client $n$, \textit{(b)} a vector space $\mathcal{F}(e)$ for each $e\in \calE_s$, and \textit{(c)} a linear map $\mathcal{F}_{n\unlhd e}:\mathcal{F}{(n)}\to \mathcal{F}{(e)}$ for each incident node-edge pair $n\unlhd e$. The map $\calF_{n\unlhd e}$ is also known as a restriction map, which transports an element from the node space to the edge space. Therefore, for $(n,m)\in\calE_s$, we can compute the difference between $\bx_n\in \calF(n)$ and $\bx_m\in\calF(m)$ via the restrictions $\calF_{n\unlhd e} \bx_n$ and $\calF_{m\unlhd e} \bx_m$ (both in $\calF((n,m))$), capturing the relationship between two nodes.

For $\bX = (\bx_n)_{n\in [N]}\in \bigoplus_{n\in[N]} \calF(n)$, the sheaf Laplacian of $(G_s,\mathcal{F})$ is a linear map that maps $\bX$ to $L_{\calF}(\bX)\in \bigoplus_{n\in[N]} \calF(n)$, whose node or client $n$ component is defined as
\begin{align}\label{eq.L_F_xn}
L_{\mathcal{F}}(\bX)_{n} := \sum_{n,m\unlhd e} \mathcal{F}_{n\unlhd e}\T (\mathcal{F}_{n\unlhd e}\bx_{n}- \mathcal{F}_{m\unlhd e}\bx_{m}). 
\end{align}
Let $\bD$ be the block-diagonal of $L_{\mathcal{F}}$. The normalized sheaf Laplacian is given by
\begin{align}
\Delta_{\mathcal{F}}: = \bD^{-1/2} L_{\mathcal{F}} \bD^{-1/2}. \label{eq.delta_F}
\end{align}
This operator can be used to describe the process of vector space elements being transported and rotated into neighboring vector spaces, capturing the essence of element transitions within the sheaf. 

We model the restriction map $\calF_{n\unlhd e}$ using a \gls{MLP}, which is updated with each server gradient update instance. We stack $T_s$ layers of the sheaf diffusion model to evolve the client embeddings as 
shown in Fig.~\ref{fig_NSD_network}. Denote the sheaf diffusion model at layer $t\in [T_s]$ as
\begin{align}
g(\cdot,\btheta _{t})= - \sigma \parens*{ \Delta_{\mathcal{F}^{(t)}} (\bI \otimes \bW_{1}^{(t)} ) (\cdot) \bW_{2}^{(t)} },
\end{align}
where $\btheta_t$ is the set of learnable parameters at layer $t$, including the learnable weight matrices $\bW_{1}^{(t)}$ and $\bW_{2}^{(t)}$ and the MLP weights for the restriction maps.

\begin{figure}[H]
\centering
\includegraphics[width=\columnwidth]{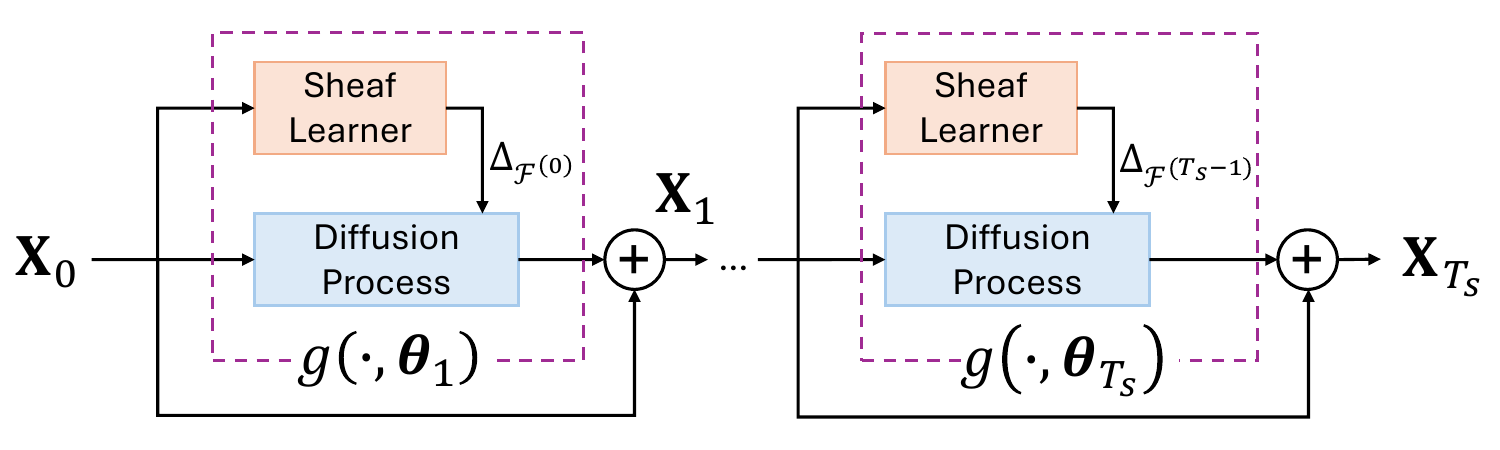}
\caption{Sheaf diffusion model pipeline $S(\cdot;\btheta)$. The evolution of the sheaf is described by a learnable function, enabling the model to utilize the latest available features to manipulate the underlying geometry of the graph.} 
\label{fig_NSD_network}
\end{figure}

The corresponding sheaf Laplacian is denoted as $\Delta_{\calF\tc{t}}$.
The server adopts a diffusion-type model \cite{bodnar2022neural} to update the client graph-level embeddings as follows: for $t=1,\dots,T_s$, 
\begin{align}\label{sheaf_diffusion}
\bX_{t} = \bX_{t-1} + g(\bX_{t-1};\btheta_{t-1}),
\end{align}
where $\sigma(\cdot)$ denotes a non-linear function, $\bI$ is the identity matrix, and $\otimes$ is the Kronecker product.

The iteration Eq.~(\ref{sheaf_diffusion}) is performed for $T_s$ times and its output is denoted as $S(\bX_{0}; \btheta)$, given by  
\begin{align}\label{X-Ts}
\bX_{T_s} := S(\bX_{0}; \btheta) = (\bx_i\tc{T_s})_{i\in [N]},
\end{align}
where $\btheta=(\btheta_t)_{t\in[T_s]}$. 

The collaboration graph $G_{s}$ evolves in tandem with client model training, with graph-level embeddings progressively refined to better represent client graph features. This process seamlessly integrates into client model training, requiring no additional models, and aligns the evolution of $G_{s}$ with the adaptation of client models.

\subsection{Optimization of the Hypernetwork}
To generate personalized model parameters for each client, we utilize HNs taking the updated client graph-level embeddings $\bX_{T_s}$ from Eq.~(\ref{X-Ts}) as input. The HNs is optimized to generate personalized backbone parameters for all clients using the collaboration graph feature matrix.

Unlike conventional HNs that use a basic \gls{MLP} \cite{zhang2018graph,shamsian2021personalized,xu2023heterogeneous} without explicit client relationships, using the embeddings from the neural sheaf diffusion model allows FedSheafHN to integrate potential latent relationships between the clients. 
To achieve this, the HNs denoted by $H(\cdot; \bm{\varphi})$ for our FedSheafHN incorporates an attention layer on top of a MLP, where the attention-based embedding is given by
\begin{align}\label{hn_att}
\bX_{T_s}^{\mathrm{att}}
& = f_{\mathrm{softmax}}\parens*{ (\bX_{T_s} \bA_{\mathrm{att}}^Q )(\bX_{T_s} \bA_{\mathrm{att}}^K )\T } \bX_{T_s} \bA_{\mathrm{att}}^V ,
\end{align}
$f_{\mathrm{softmax}}$ denotes the row-wise softmax function, and $\bA_{\mathrm{att}}^Q$, $\bA_{\mathrm{att}}^K$, and $\bA_{\mathrm{att}}^V$ are all learnable matrices. 
Then, $\bX_{T_s}^{\mathrm{att}}$ is fed into the \gls{MLP} to generate the backbone parameters. 

In summary, the output of the HNs is given by
\begin{align}\label{Omega-b-s}
\bOmega_b\tc{s} := H(\bX_{T_s}; \bm{\varphi}) = (\bomega_{b,i}\tc{s})_{i \in [N]},
\end{align}
where $\bomega_{b,i}\tc{s}$ denotes the backbone parameters of client $i$, and $\bm{\varphi}$ are the HNs parameters, including attention layer and MLP.

\subsection{Learning Procedure}

After receiving $\bomega_{b,i}\tc{s}$ from the server in the current communication round, each client $i$ performs $T_c$ rounds of local training to obtain the updated parameters $\bomega_i\tc{c} = (\bomega_{b,i}\tc{c},\bomega_{h,i}\tc{c})$. In each local training epoch, the model parameters $\bomega_i$ are updated using $\bomega_i := \bomega_i - \gamma \nabla_{\bomega_i}\mathcal{L}_i(\bomega_i)$, where $\gamma$ is a hyperparameter. Each client $i$ transmits $\Delta \bomega_{b,i}:= \bomega_{b,i}^{(c)}-\bomega_{b,i}^{(s)}$ to the server, forming $\Delta \bOmega_b=(\Delta \bomega_{b,i})_{i \in [N]}$. 

To simplify notations, we represent our quantities of interest in vector form. We use $\tilde{\btheta}$ and $\tilde{\bm{\varphi}}$ to represent the vectorized versions of $\btheta$ and $\bm{\varphi}$, respectively.
The vectorized version of $\bX_{T_s}$ in Eq.~(\ref{X-Ts}) is written as $\tilde{\bX}_{T_s} := S( \tilde{\bX}_{0}; \tilde{\btheta})$.
Similarly, the vectorized version of $\bOmega_b\tc{s}$ in Eq.~(\ref{Omega-b-s}) is given by $\tilde{\bOmega}_b\tc{s} := H(\tilde{\bX}_{T_s}; \tilde{\bm{\varphi}} )$.

Using Eq.~(\ref{loss_func}), we derive the gradients of FedSheafHN for the parameters $\tilde{\btheta}, \tilde{\bm{\varphi}}$ on the server via $\Delta \tilde{\bOmega}_b$ and the chain rule. 
From \cite{shamsian2021personalized}, we have the update rule for the server's parameters: 
\begin{align}
&\Delta \tilde{\btheta} =\mathbf{J}_{\tilde{\btheta}}(\tilde{\bOmega}_b\tc{s}) \Delta \tilde{\bOmega}_b \label{chain_rule_1} \\
&\Delta \tilde{\bm{\varphi}} = \mathbf{J}_{\tilde{\bm{\varphi}} }(\tilde{\bOmega}_b\tc{s}) \Delta \tilde{\bOmega}_b.\label{chain_rule_2} 
\end{align}
We also have
\begin{align}
\label{gd_s}
\mathbf{J}_{\tilde{\btheta}}{(\tilde{\bOmega}_b\tc{s})} =\mathbf{J}_{\tilde{\btheta}} (S(\tilde{\bX}_{0} ; \tilde{\btheta})) \mathbf{J}_{\tilde{\bX}_{T_s}}(H(\tilde{\bX}_{T_s};\tilde{\bm{\varphi}})), 
\end{align}
as well as 
\begin{align}
\label{gd_h}
\mathbf{J}_{\tilde{\bm{\varphi}} }{(\tilde{\bOmega}_b\tc{s})} = \mathbf{J}_{\tilde{\bm{\varphi}}}(H(\tilde{\bX}_{T_s};\tilde{\bm{\varphi}})),
\end{align}
where $\bJ_{\tilde{\btheta}}(\cdot)$ and $\bJ_{\tilde{\bm{\varphi}}}(\cdot)$ are the Jacobians \gls{wrt} $\tilde{\btheta}$ and $\bm{\tilde{\varphi}}$, respectively.
Throughout the training process, the parameters $\tilde{\btheta}$ of the model $S$ are updated according to Eq.~(\ref{chain_rule_1}) and Eq.~(\ref{gd_s}) using the gradient descent step $\tilde{\btheta} := \tilde{\btheta} - \alpha  \Delta \tilde{\btheta}$, where $\alpha$ is a gradient control hyperparameter.
The parameters $\tilde{\bm{\varphi}}$ of the model $H$ are updated according to Eq.~(\ref{chain_rule_2}) and Eq.~(\ref{gd_h}) using $\tilde{\bm{\varphi}} := \tilde{\bm{\varphi}} - \beta \Delta \tilde{\bm{\varphi}}$, where $\beta$ is a hyperparameter. 
The full procedure of FedSheafHN is outlined in Algorithm~\ref{alg_fedsheafhn} in Appendix~\ref{app_algorithm}. 

\subsection{Theoretical Basis}
\paragraph{Convergence Guarantees}
We provide guarantees by bounding the expected average gradient norm, with the detailed proof and complete formulation presented in Appendix~\ref{app_convergence-proofs}.
\begin{theorem}
\label{thm_convergence}
Under standard assumptions of smoothness and boundedness (cf.\ the Appendix), the optimization of FedSheafHN after $T$ communication round satisfies
\begin{align}
& \frac{1}{T}\sum_{t=1}^{T}\mathbb{E}\big\|\nabla \mathcal{L}(\omega_t)\big\|^2\le \frac{(\mathcal{L}(\omega_0)-\mathcal{L}^{\ast})}{\sqrt{NT}}+\frac{LL_1(16\sigma_1^2+8T_cL_G^2)}{T_c\sqrt{NT}} \nn
&\qquad\qquad\qquad\qquad\quad +\frac{16NL_7^2b_3^2b_2^2}{T}+\frac{(L+2)}{T_c}b_1^2 ,
\end{align}
where $\mathcal{L}$ denotes the global objective Eq.~(\ref{loss_func}) with lower bound $\mathcal{L}^\ast$. $\omega_0$ and $\omega_t$ are the initial and $t$-th communication round parameters. $L$, $L_1$, $L_7$ are smoothness constants. $b_1$, $b_2$, $b_3$ bound gradient norms. $\sigma_1$ bounds gradient variance, and $L_G$ measures client dissimilarity.
\end{theorem}

\paragraph{Generalization Guarantees}
From Eq.~(\ref{loss_func}), the expected loss is $\mathcal{L}$, and the empirical loss is denoted as $\hat{\mathcal{L}}$. All learnable parameters are assumed to be bound within a ball of radius $R$ and satisfy global Lipschitz conditions for all potential functions in $\mathcal{L}$. The parameter space dimensions for $\bx_i\tc{c}$, $\btheta$, $\bm{\varphi}$, and $\bomega_{h,i}$ are $d_{i}$, $d_{\btheta}$, $d_{\bm{\varphi}}$, and $d_{\bomega_{h,i}}$, respectively. 
$L$, $L_s$, and $L_H$ are constants related to Lipschitz conditions.
The proof is provided in Appendix~\ref{app_generalization-proofs}. 

\begin{theorem}
\label{thm_generalization_bound}
Suppose there are $N$ clients, and each client $i$ has a sample size larger than $M$, with $M=$
$\mathcal{O}\left(\big(d_{i}+\frac{d_{\btheta}+d_{\bm{\varphi}}+d_{\bomega_{h,i}}}{N}\big)\frac{1}{\xi^2}\log{(\frac{RL(L_s+L_H+1)}{\xi})
+\frac{1}{N\xi^2}\log{\frac{1}{\delta}}}\right)$,
it holds with probability at least $1-\delta$ that $\left |\mathcal{L}(\omega_i)-\hat{\mathcal{L}}(\omega_i)\right |\le\xi$.
\end{theorem} 

\begin{table*}[!t] 
\caption{Results on the non-overlapping node scenario. The reported results are mean and standard deviation over five different runs. The best and the second-best results are highlighted in bold and underlined, respectively. OOM refers to out-of-memory. }
\label{non-overlapping}
\begin{center}
\resizebox{0.9\textwidth}{!}{
\tiny
\renewcommand{\arraystretch}{1.2}
\begin{tabular}{lcccccc}
\toprule
\multirow{2}{*}{Methods} & \multicolumn{2}{c}{Cora}            & \multicolumn{2}{c}{Citeseer}        & \multicolumn{2}{c}{Pubmed}   \\ 
\cline{2-7}
                & 10 Clients & 20 Clients & 10 Clients & 20 Clients & 10 Clients & 20 Clients\\ 
\hline
Local           & 71.26$\pm$0.29   & 74.65$\pm$0.42    & 67.82$\pm$0.13    & 65.98$\pm$0.17  & 82.81$\pm$0.39  & 82.65$\pm$0.03  \\ \hline
FedAvg          & 72.38$\pm$2.45   & 69.81$\pm$13.28   & 65.71$\pm$0.37    & 63.08$\pm$7.38  & 79.88$\pm$0.07  & 78.48$\pm$7.65 \\
FedProx         &60.18$\pm$7.04    &48.22$\pm$6.81     &63.33$\pm$3.25     &64.85$\pm$1.35   &82.55$\pm$0.24   &80.50$\pm$0.25 \\
FedSage+        & 69.05$\pm$1.59   & 57.97$\pm$ 12.6   & 65.63$\pm$3.10    & 65.46$\pm$0.74  & 82.62$\pm$0.31  & 80.82$\pm$0.25  \\
FGGP            &66.57$\pm$0.56    &54.58$\pm$0.83     &60.99$\pm$0.68     &57.13$\pm$0.33   &76.84$\pm$0.59   &73.08$\pm$0.81 \\
\hline
FedPer          & 65.24$\pm$1.19  & 70.80$\pm$0.72  & 62.75$\pm$5.65  & 60.06$\pm$2.24   & 71.42$\pm$2.17  & 73.05$\pm$0.36    \\
pFedHN          & 65.24$\pm$0.21  & 70.65$\pm$2.21  & 63.45$\pm$0.44  & 58.98$\pm$1.60   & 77.22$\pm$1.74  & 73.20$\pm$2.47      \\
pFedGraph       & 76.57$\pm$1.37  & 76.61$\pm$0.63  & 71.61$\pm$0.64  & 67.43$\pm$0.49  & 80.38$\pm$0.30  & 80.48$\pm$0.58      \\
FED-PUB         & 80.25$\pm$0.54 & 81.04$\pm$0.31        & \udcloser{73.42$\pm$0.75}        & 66.87$\pm$0.21   & \udcloser{85.95$\pm$0.77}  & \udcloser{85.38$\pm$0.82}   \\
Flow            & \udcloser{81.81$\pm$0.14}   &\udcloser{81.42$\pm$0.18}  &73.05$\pm$0.24   & \udcloser{70.01$\pm$0.24}      &85.53$\pm$0.05             &84.48$\pm$0.09 \\
FedSheafHN (ours)& \textbf{83.49$\pm$0.22}  & \textbf{82.35$\pm$0.18}  & \textbf{75.27$\pm$0.39}  & \textbf{72.20$\pm$0.24}   & \textbf{86.50$\pm$0.09}  & \textbf{85.57$\pm$0.05}    \\ 
\bottomrule 
\toprule
\multirow{2}{*}{Methods} & \multicolumn{2}{c}{Amazon-Computer} & \multicolumn{2}{c}{Amazon-Photo}    & \multicolumn{2}{c}{ogbn-arxiv}        \\ 
\cline{2-7} 
                & 10 Clients & 20 Clients & 10 Clients & 20 Clients & 10 Clients & 20 Clients   \\ \hline
Local           & 85.90$\pm$0.35 & 87.79$\pm$0.49  &  76.21$\pm$1.34  &  81.93$\pm$0.59 & 64.92$\pm$0.09 &  65.06$\pm$0.05  \\ 
\hline
FedAvg         & 66.78$\pm$0.00  &  71.44$\pm$0.08  &  79.61$\pm$3.12  &  82.12$\pm$0.02    &  48.77$\pm$2.88   &     42.02$\pm$17.09  \\
FedProx        &83.81$\pm$1.09 &73.05$\pm$1.30 &80.92$\pm$4.64 &82.32$\pm$0.29 &64.37$\pm$0.18 &63.03$\pm$0.04 \\ 
FedSage+       & 80.50$\pm$1.30   & 70.42$\pm$0.85  & 76.81$\pm$8.24  & 80.58$\pm$1.15  & 64.52$\pm$0.14  & 63.31$\pm$0.20      \\
FGGP           &66.91$\pm$1.24 &48.25$\pm$0.85 &68.52$\pm$0.53 &54.85$\pm$0.57 & OOM & OOM \\
\hline
FedPer         & 67.41$\pm$0.67  & 72.91$\pm$0.50   & 76.10$\pm$2.18   &82.80$\pm$0.07    &45.31$\pm$0.33   &  47.55$\pm$0.01        \\
pFedHN         & 66.85$\pm$0.09  & 69.94$\pm$1.27  & 74.12$\pm$0.90  & 82.03$\pm$0.10  & 57.70$\pm$0.30  & 50.61$\pm$0.85      \\
pFedGraph      & 66.45$\pm$0.83  & 71.57$\pm$0.36  & 74.57$\pm$1.05 & 84.04$\pm$0.50  & 56.37$\pm$0.31  & 56.19$\pm$0.83      \\
FED-PUB        &88.00$\pm$0.24 & 89.99$\pm$0.05  & 93.73$\pm$0.22  & 91.04$\pm$0.32     &66.48$\pm$0.09           &66.99$\pm$0.18  \\
Flow           &\udcloser{88.67$\pm$0.12} &\udcloser{90.31$\pm$0.11} &\udcloser{93.94$\pm$0.07} &\udcloser{92.24$\pm$0.10} &\udcloser{68.91$\pm$0.04} &\udcloser{69.89$\pm$0.11} \\
FedSheafHN (ours) & \textbf{90.56$\pm$0.03} & \textbf{91.00$\pm$0.09} & \textbf{94.22$\pm$0.10}  & \textbf{92.99$\pm$0.05}   & \textbf{71.28$\pm$0.03}  & \textbf{71.75$\pm$0.09} \\ 
\bottomrule
\end{tabular}}
\end{center}
\end{table*}
\section{Experiments}\label{sect_exper}
We evaluate the performance of FedSheafHN on six datasets, focusing on node classification tasks within two distinct subgraph FL scenarios.  
We compare its performance with several state-of-the-art baselines using Federated Accuracy \cite{shamsian2021personalized} as the metric, defined as $\frac{1}{N} \sum_{i \in N} \mathrm{Acc}(f(G_{i};\bomega_{i}))$, where $\mathrm{Acc}(\cdot)$ denotes the accuracy of the given model.
The code is available at \texttt{https://github.com/CarrieWFF/FedSheafHN}.

\subsection{Experimental Settings}
\paragraph{Real-World Datasets}
Following experimental in \cite{zhang2021subgraph,baek2023personalized}, we partition datasets into segments, assigning each client a dedicated subgraph. This allocation ensures that each participant manages a segment of a larger, original graph.
Experiments are conducted on six datasets: Cora, CiteSeer \cite{sen2008collective}, Pubmed, and ogbn-arxiv \cite{hu2020open} for citation graphs, as well as Computer and Photo \cite{mcauley2015image,shchur2018pitfalls} for Amazon product graphs.
The partitioning of these datasets is executed through the METIS graph partitioning algorithm \cite{karypis1997metis}, with the number of subsets predetermined. 
Our experiments include both the non-overlapping and overlapping node scenarios. In the non-overlapping node scenario, the METIS output is directly utilized, creating distinct subgraphs without shared nodes, thus yielding a more heterogeneous setting. 
In the overlapping node scenario, the subgraphs are achieved by sampling smaller subgraphs from the initial METIS partitioned results.
Dataset statistics are provided in Appendix~\ref{app_dataset}.

\paragraph{Baselines}
We compare our results with following baselines. The FL methods focus on training a global model, include FedAvg \cite{mcmahan2017communication}, FedProx \cite{li2020federated}, FedSage+ \cite{zhang2021subgraph}, and FGGP \cite{wan2024federated}. The PFL methods aim to provide personalized models for individual clients, include FedPer \cite{arivazhagan2019federated}, pFedHN \cite{shamsian2021personalized}, pFedGraph \cite{ye2023personalized}, FED-PUB \cite{baek2023personalized}, and Flow \cite{panchal2024flow}. 

\paragraph{Implementation Details}
\label{implementation}
The client models are two-layer GCNs for all baselines. The HN used in FedSheafHN consists of an attention layer and a two-layer MLP. The hidden dimension of GCNs and HNs is set to 128. The Adam optimizer \cite{kinga2015method} is applied for optimization. Data splits are 40\% for training, 30\% for validation, and 30\% for testing. FL runs for 100 communication rounds on Cora, CiteSeer, Pubmed, and Computer, and 150 rounds on Photo and ogbn-arxiv, with 3 local epochs for Cora/CiteSeer, 15 for Pubmed, and 10 for the rest.
In our framework, server updates $G_s$ every $r_{in}=5$ communication rounds.
Hyperparameters used are given in Appendix~\ref{app_hyperparameters}. 

All experiments are conducted using PyTorch \cite{paszke2019pytorch} and PyTorch Geometric \cite{fey2019fast} on NVIDIA GeForce RTX 3090. 
The runtime of FedSheafHN depends mainly on the number of clients, as client training is performed sequentially for simulation. Generally, training 50 clients with 3 local epochs over 100 total rounds takes around 3 hours, with memory usage varying from 300 to 2500 MB across datasets.

\subsection{Experimental Results}
\begin{table*}[!tb]
\caption{Results on the overlapping node scenario. The reported results are mean and standard deviation over five different runs. The best and the second-best results are highlighted in bold and underlined, respectively. OOM refers to out-of-memory.}
\label{overlapping}
\begin{center}
\resizebox{0.9\textwidth}{!}{
\tiny
\renewcommand{\arraystretch}{1.2}
\begin{tabular}{lcccccc}
\toprule
\multirow{2}{*}{Methods} & \multicolumn{2}{c}{Cora}   & \multicolumn{2}{c}{Citeseer}   & \multicolumn{2}{c}{Pubmed}  \\ \cline{2-7} 
                & 30 Clients & 50 Clients & 30 Clients & 50 Clients & 30 Clients & 50 Clients   \\ \hline
Local           & 71.65$\pm$0.12  & 75.45$\pm$0.51 & 64.54$\pm$0.42  & 66.68$\pm$0.44 & 80.72$\pm$0.16  & 80.54$\pm$0.11    \\ \hline
FedAvg          & 63.84$\pm$2.57  & 57.98$\pm$0.06  & 66.11$\pm$1.50  & 58.00$\pm$ 0.29 & 83.11$\pm$0.03  & 82.24$\pm$0.73      \\
FedProx         &51.38$\pm$1.74 &56.27$\pm$9.04 &66.11$\pm$0.75 &66.53$\pm$0.43 &82.13$\pm$0.13 &80.50$\pm$0.46 \\
FedSage+        & 51.99$\pm$0.42  & 55.48$\pm$11.5  & 65.97$\pm$0.02  & 65.93$\pm$0.30  & 82.14$\pm$0.11  & 80.31$\pm$0.68      \\
FGGP            &62.59$\pm$0.49 &53.50$\pm$0.48 &57.50$\pm$0.39 &52.90$\pm$0.53 &71.70$\pm$0.82 &70.49$\pm$0.43 \\
\hline
FedPer          & 54.70$\pm$2.58  & 64.66$\pm$0.07  & 58.91$\pm$2.41  & 58.33$\pm$2.91  & 70.08$\pm$0.38  & 71.13$\pm$0.04     \\
pFedHN          & 48.71$\pm$2.19  & 49.19$\pm$2.54  & 54.67$\pm$1.28  & 46.34$\pm$2.24  & 69.21$\pm$1.37  & 65.55$\pm$1.41      \\
pFedGraph       & \udcloser{77.72$\pm$0.41}  & \udcloser{77.69$\pm$0.20}  & 69.60$\pm$0.11  & \udcloser{67.84$\pm$0.75}  & 83.12$\pm$0.37  & 82.60$\pm$0.27      \\
FED-PUB         & 75.66$\pm$1.02  & 76.52$\pm$0.87  & 68.76$\pm$0.26  & 66.95$\pm$0.70  & \textbf{84.97$\pm$0.30}  & \textbf{84.38$\pm$0.38}      \\
Flow            &77.62$\pm$0.17 &77.45$\pm$0.21 &\udcloser{70.86$\pm$0.10} &67.39$\pm$0.15 &82.50$\pm$0.07 &82.11$\pm$0.23 \\
FedSheafHN (ours)& \textbf{80.30$\pm$0.11}  & \textbf{78.06$\pm$0.22}  & \textbf{71.90$\pm$0.15}  & \textbf{68.59$\pm$0.12}   & \udcloser{84.45$\pm$0.03}  & \udcloser{83.65$\pm$0.02}    \\ 
\bottomrule
\toprule
\multirow{2}{*}{Methods} & \multicolumn{2}{c}{Amazon-Computer} & \multicolumn{2}{c}{Amazon-Photo}    & \multicolumn{2}{c}{ogbn-arxiv}        \\ 
                \cline{2-7} 
                & 30 Clients & 50 Clients & 30 Clients & 50 Clients & 30 Clients & 50 Clients  \\ \hline
Local           & 64.65$\pm$1.09 &  68.28$\pm$1.02    &   72.60$\pm$0.78  & 80.68$\pm$1.48   &  61.32$\pm$0.04 &60.04$\pm$0.04      \\ \hline
FedAvg          & 68.99$\pm$3.97 &  67.69$\pm$0.12    &  83.74$\pm$0.72 &  75.59$\pm$0.06   &   49.91$\pm$2.84  & 52.26$\pm$1.02\\
FedProx &83.84$\pm$0.89 &76.60$\pm$0.47 &89.17$\pm$0.40 &72.36$\pm$2.06 &59.86$\pm$0.16 &61.12$\pm$0.04 \\ 
FedSage+        & 81.33$\pm$1.20  & 76.72$\pm$0.39  & 88.69$\pm$0.99  & 72.41$\pm$1.36  & 59.90$\pm$0.12  & 60.95$\pm$0.09      \\
FGGP &70.44$\pm$1.65 &64.04$\pm$1.25 &75.01$\pm$0.51 &62.40$\pm$1.50 & OOM & OOM \\
\hline
FedPer          & 62.06$\pm$0.08 &  67.87$\pm$0.08    & 68.82$\pm$0.34 &  75.69$\pm$0.14  & 42.19$\pm$0.01 & 43.77$\pm$0.33               \\
pFedHN          & 59.88$\pm$1.31  & 59.63$\pm$2.42  & 67.85$\pm$0.23  & 72.07$\pm$1.52  & 43.64$\pm$0.40  & 41.38$\pm$0.66      \\
pFedGraph       & 76.77$\pm$1.53  & 74.46$\pm$0.42  & 77.07$\pm$0.69  & 87.81$\pm$1.96  & 59.29$\pm$0.39  & 58.96$\pm$0.56      \\ 
FED-PUB         & \udcloser{88.40$\pm$0.80} &  88.77$\pm$0.07    & 92.01$\pm$0.07 &  91.76$\pm$0.35  & 63.96$\pm$0.10 &  64.66$\pm$0.19   \\ 
Flow            &87.61$\pm$0.10 &\udcloser{88.83$\pm$0.08} &\udcloser{92.74$\pm$0.05} &\udcloser{91.88$\pm$0.04} &\udcloser{65.84$\pm$0.03} &\udcloser{65.86$\pm$0.03} \\
FedSheafHN (ours)& \textbf{89.25$\pm$0.02} & \textbf{89.29$\pm$0.05}  & \textbf{93.34$\pm$0.04}  & \textbf{92.36$\pm$0.06}   & \textbf{67.69$\pm$0.10}  & \textbf{67.51$\pm$0.05}    \\ \bottomrule
\end{tabular}
}
\end{center}
\end{table*}

\paragraph{Non-Overlapping Scenario}
In Table~\ref{non-overlapping}, we present the node classification results for the non-overlapping scenario, characterized by a notably heterogeneous subgraph FL challenge. FedSheafHN stands out as it consistently outperforms other baseline approaches. 
FED-PUB, as our primary target baseline for subgraph PFL, achieves commendable results by identifying community structures via similarity estimation and selectively filtering irrelevant weights from diverse communities. However, in scenarios with many distinct clients and extreme heterogeneity, relying solely on similarity measures is insufficient for adequately inferring client relationships.
Flow dynamically selects local and global model parameters based on the input instance. Although it performs well, it is not specifically designed for subgraph FL and does not fully exploit the underlying relationships among clients.
FedSheafHN enhances relationship inference and feature representation in highly heterogeneous scenarios, thus boosting the performance of HNs in creating personalized client models. Its adaptability to diverse and complex cases makes it a promising choice for addressing complexities in tough scenarios.

\paragraph{Overlapping Scenario }
The overlapping scenario, with lower heterogeneity as outlined in Table~\ref{overlapping}, still showcases FedSheafHN's effectiveness across various datasets. Although its advantages are less pronounced compared to the non-overlapping scenario, the diverse feature representation of FedSheafHN is designed to address the intricacies of this setting. It consistently outperforms state-of-the-art baselines, highlighting its robustness and adaptability to varying degrees of heterogeneity, positioning it as a more reliable option.
\paragraph{Convergence Plots}
Fig.~\ref{conv_ov} illustrates the fast convergence of FedSheafHN, likely due to its ability to discern client cooperation and enable intelligent information sharing through data-driven client descriptions. Furthermore, the HNs, with learned attention and batch processing of all clients, efficiently generate model parameters and integrate global information.

\paragraph{Effectiveness of Personalization}
Our goal is to improve the accuracy of all clients by generating personalized models through sheaf collaboration. Effective collaboration should benefit all clients, regardless of heterogeneity in local subgraphs. To evaluate this, we measure the standard deviation of local accuracies across clients, where a lower value indicates more effective personalization. As shown in Table~\ref{eff_per_table}, FedSheafHN achieves the lowest standard deviation, demonstrating a more stable and balanced performance across all clients.
\begin{table}[ht!]
\caption{Mean and standard deviation of local accuracies across clients on the ogbn-arxiv dataset.}
\vspace{3mm}
\label{eff_per_table}
\centering
\begin{tabular}{lcc}
\toprule
\multirow{2}{*}{Methods} & non-overlapping & overlapping \\ 
                         & 20 clients      & 30 clients   \\ 
\midrule
pFedGraph                  & 56.01$\pm$16.10  & 59.13$\pm$12.60    \\
FED-PUB                    & 66.67$\pm$10.48  & 63.94$\pm$10.43    \\
Flow                       & 69.78$\pm$9.06   & 65.84$\pm$9.15    \\
FedSheafHN (ours)          & 71.88$\pm$7.92   & 67.78$\pm$7.98    \\ 
\bottomrule
\end{tabular}
\vspace{1mm}
\end{table}

\begin{figure}[!ht]
\centering
\includegraphics[width=\columnwidth]{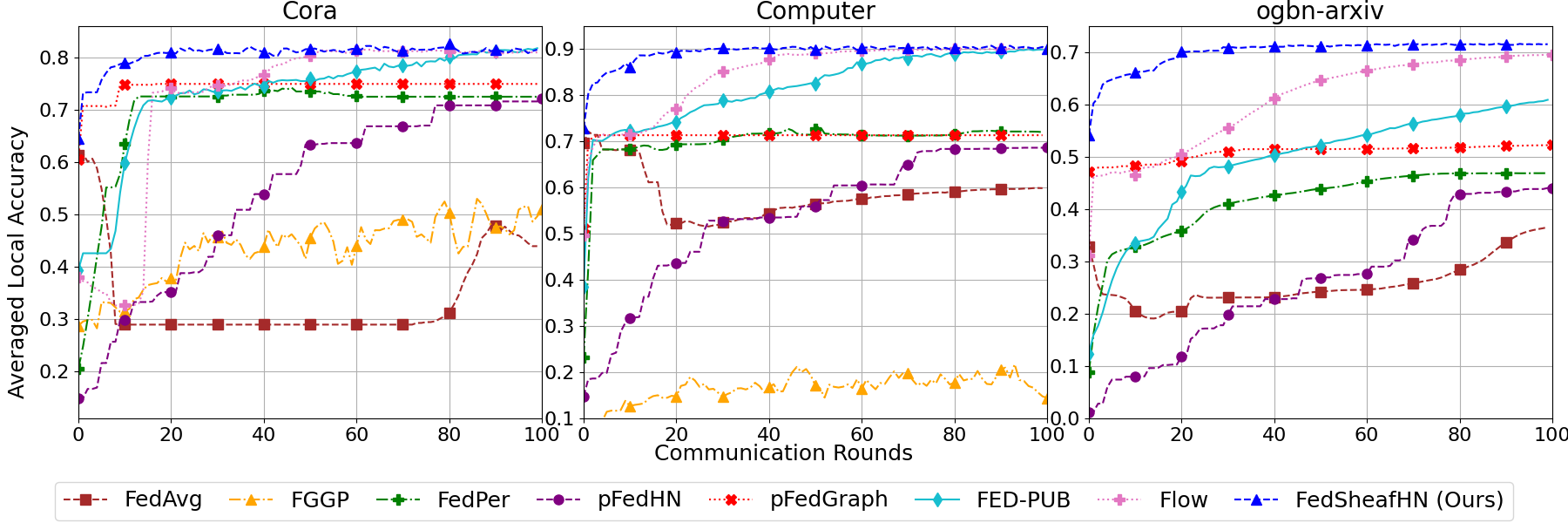}
\caption{Convergence plots in the non-overlapping scenario (20 clients) across three datasets.}
\label{conv_ov}
\vspace{6mm}
\end{figure}

\begin{figure}[!hbt] 
\centering
\includegraphics[width=\columnwidth]{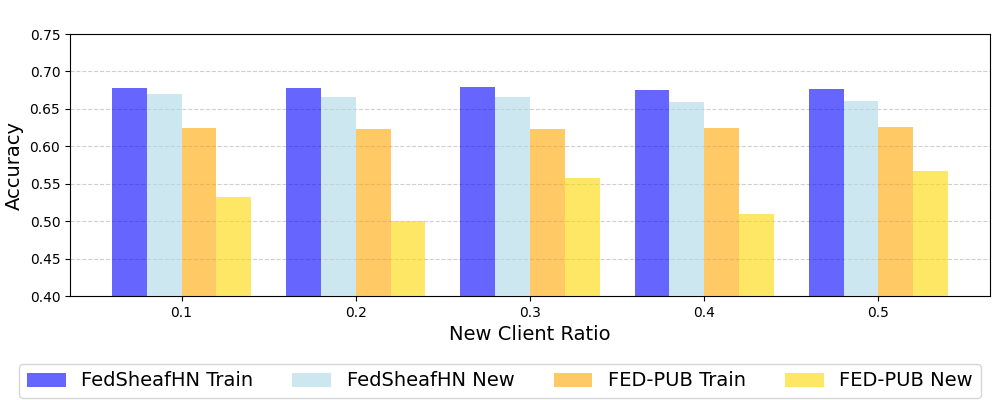}
\caption{Generalization results on ogbn-arxiv dataset with 30 clients in overlapping scenario. "Train" indicates average test accuracy for trained clients, and "New" for newly joined clients with new client ratios from 0.1 to 0.5. }
\label{fig_newclients}
\vspace{1mm}
\end{figure}

\paragraph{Generalization to New Clients}
We evaluate FedSheafHN on unseen clients without retraining or fine-tuning the shared model. Once $S(\cdot; \btheta)$ and $H(\cdot; \bm{\varphi})$ are trained, we simply freeze $\btheta$ and $\bm{\varphi}$, initialize the new client's model, and train locally to derive graph-level embedding $\bx_{new}$.
As shown in Fig.~\ref{fig_newclients}, FedSheafHN achieves comparable accuracy for new clients with minimal performance reduction, requiring only one communication round for $\bx_{new}$.

\begin{table}[tb!]
\caption{Ablation studies of FedSheafHN on the ogbn-arxiv dataset.}
\vspace{3mm}
\label{ab_ogbn_table}
\centering
\begin{tabular}{lcc}
\toprule
\multirow{2}{*}{Methods} & non-overlapping & overlapping \\ 
                         & 20 clients      & 30 clients   \\ 
\midrule
Base model (FedAvg)                  & 36.56  & 42.95    \\
+ Hypernetwork                       & 61.26  & 56.53    \\
+ Collaboration graph                & 68.51  & 64.68    \\
+ Sheaf diffusion                    & 70.05  & 65.87    \\
+ Attention                          & 70.97  & 67.09    \\
+ Dynamic embedding (ours)           & 71.87  & 67.89    \\ 
\bottomrule
\end{tabular}
\end{table}

\begin{table}[t!]
\caption{Different \gls{GNN} models on collaboration graph.}
\vspace{3mm}
\label{ab_ogbn_table2}
\centering
\begin{threeparttable}
\begin{tabular}{lcc}
\toprule
\multirow{2}{*}{Methods} & non-overlapping & overlapping \\ 
                         & 20 clients      & 30 clients \\ 
\midrule
w/o Sheaf diffusion\tnote{a} & 69.75 & 65.66 \\
w/ GCN\tnote{b}               & 69.86 & 65.96 \\
w/ GAT                        & 69.92 & 65.90 \\
w/ Sheaf diffusion (ours)     & 71.87 & 67.89 \\ 
\bottomrule
\end{tabular}

\begin{tablenotes}
    \scriptsize
    \item[a] Remove the ``sheaf diffusion'' part in FedSheafHN.
    \item[b] Replace sheaf diffusion model with GCN.
\end{tablenotes}
\end{threeparttable}
\end{table}

\subsection{Ablation Studies}
\paragraph{Effect of Each Module}
Table~\ref{ab_ogbn_table} shows our ablation study, adding components to the base model (FedAvg) to assess their impact. 
The results highlight the positive contributions of constructing collaboration graph (collaboration graph), dynamic embedding updates during training (dynamic embedding), sheaf diffusion for graph enhancement (sheaf diffusion), an improved batch-wise HNs (hypernetwork), and incorporating an attention layer in the HNs (attention).
All components improve performance, with the collaboration graph and enhanced hypernetwork being most impactful. Notably, replacing the collaboration graph with one-hot client vectors in the ``+ Hypernetwork'' variant results in noticeably performance drop.
The collaboration graph with client embeddings refines representations, enabling HNs to generate tailored model parameters.

\paragraph{Effect of Alternative GNNs}
Table~\ref{ab_ogbn_table2} shows that removing or replacing sheaf diffusion with other GNNs consistently results in lower performance. This highlights the effectiveness of sheaf diffusion in enhancing the collaboration graph by uncovering latent relationships.

\paragraph{Effect of Varying Local Epochs}
Fig.~\ref{fig_localep} shows that increasing local update steps $T_c$ can cause local models diverging towards their subgraphs, highlighting the trade-off between local training and convergence. 
Thus, more local epochs do not always improve performance, and the optimal number varies by dataset.

\subsection{Further Discussions}
\paragraph{Computational Overhead}
Both sheaf collaboration and HNs run entirely on the server, so clients incur no extra computation compared to standard FL baselines. The only added communication is graph-level embedding, which introduces negligible overhead. As shown in Table~\ref{train_time}, the server-side training time introduced by sheaf diffusion and HNs is modest and brings clear performance gains. 

Moreover, increasing the number of clients does not add new trainable parameters on the server. The server training time increases by just 0.34\% and 4.42\% when scaling from 20 to 50 and 100 clients, respectively, indicating good scalability.

\paragraph{Robustness to Malicious Clients}
Fig.~\ref{fig_attack} illustrates scenarios where malicious clients attempt to disrupt the training process by sending arbitrary graph-level or functional embeddings to the server, with malicious client ratios set at 0, 0.2, 0.4, and 0.6. Two types of embeddings are considered for the attack: (1) same-value parameters generated by $\bx_{i}^{(ma)}=a\bI_{d_{{\bx}_{i}}}$, where $\bI_{d_{\bx_{i}}}\in \mathbb{R}^{d_{\bx_{i}}}$ is a vector of ones and $a\sim \mathcal{N}(0,\tau ^{2})$, and (2) Gaussian noise parameters generated by $\bx_{i}^{(ma)}\sim \mathcal{N}(\mathbf{0} _{d_{\bx_{i}}},\tau^{2}\bI_{d_{\bx_{i}}})$ \cite{lin2022personalized}. Compared to the subgraph FL baseline FED-PUB, which utilizes functional embeddings of the local GNNs using random graphs as inputs to compute similarities between clients and accordingly perform model aggregation, our method demonstrates greater resilience to malicious clients.
\begin{figure}[!bt] 
\centering
\includegraphics[width=\columnwidth]{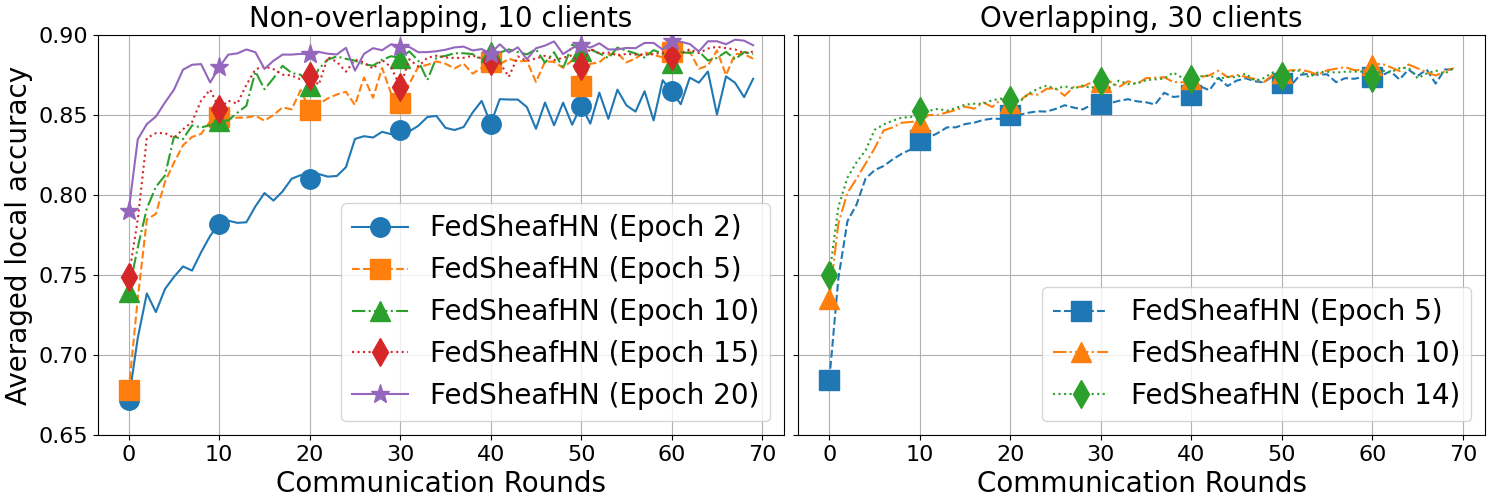}
\caption{Performance on the Computer dataset with varying local epochs.}
\label{fig_localep}
\vspace{6mm}
\end{figure}
\begin{figure}[!tb]
\centering
\includegraphics[width=\columnwidth]{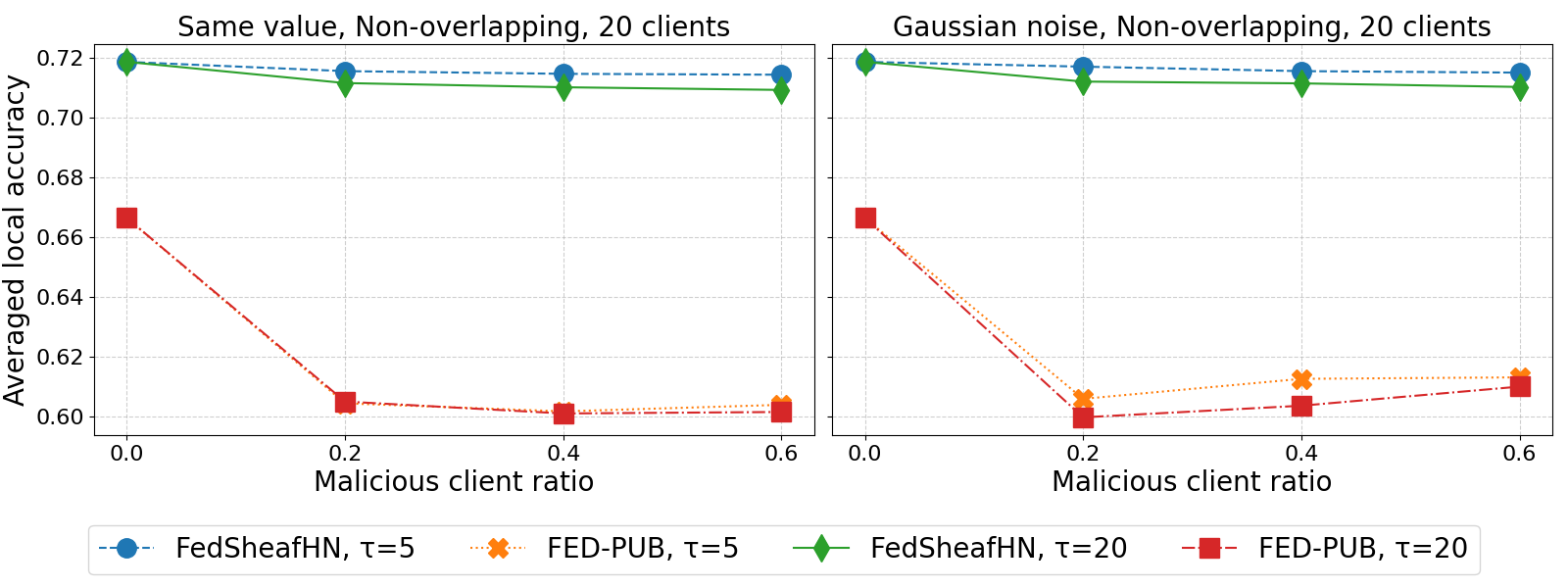}
\caption{Performance on the ogbn-arxiv dataset with malicious clients.}
\label{fig_attack}
\vspace{6mm}
\end{figure}
\begin{table}[!tb]
\caption{Server training time per communication round and total memory usage (ogbn-arxiv, 20 clients, non-overlapping scenario).}
\label{train_time}
\vspace{3mm}
\centering
\begin{tabular}{lcc}
\toprule
\textbf{Method} & \textbf{Time (s)} & \textbf{Memory (MiB)} \\
\midrule 
FedSheafHN           & 0.588  & 978 \\
w/o Sheaf diffusion  & 0.559  & 974 \\
w/ GCN               & 0.580  & 978 \\
\midrule
Base model (FedAvg)   & 0.130 & 970 \\
+ Hypernetwork       & 0.407 & 974 \\
\midrule 
50 clients           & 0.590 & 978 \\
100 clients          & 0.614 & 980 \\
\bottomrule
\end{tabular}
\end{table}

\section{Conclusion}\label{sect_conc}
We have proposed FedSheafHN, a novel framework for personalized subgraph FL that centers on sheaf collaboration and HNs to address client heterogeneity. FedSheafHN proves highly effective in both non-overlapping and overlapping scenarios. Our theoretical analysis further ensures convergence and generalization guarantees, reinforcing the practical viability of the approach. 
Experimental results demonstrate its superior performance, fast convergence, effectiveness of personalization and ability to generalize to new clients. Our ablation studies emphasize the significant contributions of each component, especially the client graph-level embedding in constructing the collaboration graph. In general, FedSheafHN provides a promising strategy for effective collaborative learning while preserving the unique data characteristics of each client within a federated network.






\bibliography{FedSheafHN_ref}

\begin{thebibliography}{50}
\providecommand{\natexlab}[1]{#1}
\providecommand{\url}[1]{\texttt{#1}}
\expandafter\ifx\csname urlstyle\endcsname\relax
  \providecommand{\doi}[1]{doi: #1}\else
  \providecommand{\doi}{doi: \begingroup \urlstyle{rm}\Url}\fi

\bibitem[Agarwal et~al.(2021)Agarwal, Kairouz, and Liu]{agarwal2021skellam}
N.~Agarwal, P.~Kairouz, and Z.~Liu.
\newblock The skellam mechanism for differentially private federated learning.
\newblock \emph{Advances in Neural Information Processing Systems}, 34:\penalty0 5052--5064, 2021.

\bibitem[Arivazhagan et~al.(2019)Arivazhagan, Aggarwal, Singh, and Choudhary]{arivazhagan2019federated}
M.~G. Arivazhagan, V.~Aggarwal, A.~K. Singh, and S.~Choudhary.
\newblock Federated learning with personalization layers.
\newblock \emph{arXiv preprint arXiv:1912.00818}, 2019.

\bibitem[Baek et~al.(2023)Baek, Jeong, Jin, Yoon, and Hwang]{baek2023personalized}
J.~Baek, W.~Jeong, J.~Jin, J.~Yoon, and S.~J. Hwang.
\newblock Personalized subgraph federated learning.
\newblock In \emph{International Conference on Machine Learning}, pages 1396--1415. PMLR, 2023.

\bibitem[Baxter(2000)]{baxter2000model}
J.~Baxter.
\newblock A model of inductive bias learning.
\newblock \emph{Journal of artificial intelligence research}, 12:\penalty0 149--198, 2000.

\bibitem[Bodnar et~al.(2022)Bodnar, Di~Giovanni, Chamberlain, Li{\`o}, and Bronstein]{bodnar2022neural}
C.~Bodnar, F.~Di~Giovanni, B.~Chamberlain, P.~Li{\`o}, and M.~Bronstein.
\newblock Neural sheaf diffusion: A topological perspective on heterophily and oversmoothing in gnns.
\newblock \emph{Advances in Neural Information Processing Systems}, 35:\penalty0 18527--18541, 2022.

\bibitem[Brock et~al.(2017)Brock, Lim, Ritchie, and Weston]{brock2017smash}
A.~Brock, T.~Lim, J.~M. Ritchie, and N.~Weston.
\newblock {SMASH}: one-shot model architecture search through hypernetworks.
\newblock \emph{arXiv preprint arXiv:1708.05344}, 2017.

\bibitem[Curry(2014)]{curry2014sheaves}
J.~M. Curry.
\newblock \emph{Sheaves, cosheaves and applications}.
\newblock University of Pennsylvania, 2014.

\bibitem[Fey and Lenssen(2019)]{fey2019fast}
M.~Fey and J.~E. Lenssen.
\newblock Fast graph representation learning with pytorch geometric.
\newblock \emph{arXiv preprint arXiv:1903.02428}, 2019.

\bibitem[Ghosh et~al.(2020)Ghosh, Chung, Yin, and Ramchandran]{ghosh2020efficient}
A.~Ghosh, J.~Chung, D.~Yin, and K.~Ramchandran.
\newblock An efficient framework for clustered federated learning.
\newblock \emph{Advances in Neural Information Processing Systems}, 33:\penalty0 19586--19597, 2020.

\bibitem[Hansen and Gebhart(2020)]{hansen2020sheaf}
J.~Hansen and T.~Gebhart.
\newblock Sheaf neural networks.
\newblock \emph{arXiv preprint arXiv:2012.06333}, 2020.

\bibitem[Hanzely et~al.(2020)Hanzely, Hanzely, Horv{\'a}th, and Richt{\'a}rik]{hanzely2020lower}
F.~Hanzely, S.~Hanzely, S.~Horv{\'a}th, and P.~Richt{\'a}rik.
\newblock Lower bounds and optimal algorithms for personalized federated learning.
\newblock \emph{Advances in Neural Information Processing Systems}, 33:\penalty0 2304--2315, 2020.

\bibitem[He et~al.(2021)He, Ceyani, Balasubramanian, Annavaram, and Avestimehr]{he2021spreadgnn}
C.~He, E.~Ceyani, K.~Balasubramanian, M.~Annavaram, and S.~Avestimehr.
\newblock Spreadgnn: Serverless multi-task federated learning for graph neural networks.
\newblock \emph{arXiv preprint arXiv:2106.02743}, 2021.

\bibitem[Hu et~al.(2020)Hu, Fey, Zitnik, Dong, Ren, Liu, Catasta, and Leskovec]{hu2020open}
W.~Hu, M.~Fey, M.~Zitnik, Y.~Dong, H.~Ren, B.~Liu, M.~Catasta, and J.~Leskovec.
\newblock Open graph benchmark: Datasets for machine learning on graphs.
\newblock \emph{Advances in Neural Information Processing Systems}, 33:\penalty0 22118--22133, 2020.

\bibitem[Huang et~al.(2021)Huang, Chu, Zhou, Wang, Liu, Pei, and Zhang]{huang2021personalized}
Y.~Huang, L.~Chu, Z.~Zhou, L.~Wang, J.~Liu, J.~Pei, and Y.~Zhang.
\newblock Personalized cross-silo federated learning on non-iid data.
\newblock In \emph{Proceedings of the AAAI conference on Artificial Intelligence}, pages 7865--7873, 2021.

\bibitem[Jiang et~al.(2019)Jiang, Kone{\v{c}}n{\`y}, Rush, and Kannan]{jiang2019improving}
Y.~Jiang, J.~Kone{\v{c}}n{\`y}, K.~Rush, and S.~Kannan.
\newblock Improving federated learning personalization via model agnostic meta learning.
\newblock \emph{arXiv preprint arXiv:1909.12488}, 2019.

\bibitem[Karypis(1997)]{karypis1997metis}
G.~Karypis.
\newblock Metis: Unstructured graph partitioning and sparse matrix ordering system.
\newblock \emph{Technical report}, 1997.

\bibitem[Kinga et~al.(2015)Kinga, Adam, et~al.]{kinga2015method}
D.~Kinga, J.~B. Adam, et~al.
\newblock A method for stochastic optimization.
\newblock In \emph{International Conference on Learning Representations}, 2015.

\bibitem[Klein et~al.(2015)Klein, Wolf, and Afek]{klein2015dynamic}
B.~Klein, L.~Wolf, and Y.~Afek.
\newblock A dynamic convolutional layer for short range weather prediction.
\newblock In \emph{Proceedings of the IEEE Conference on Computer Vision and Pattern Recognition}, pages 4840--4848, 2015.

\bibitem[Klocek et~al.(2019)Klocek, Maziarka, Wo{\l}czyk, Tabor, Nowak, and {\'S}mieja]{klocek2019hypernetwork}
S.~Klocek, {\L}.~Maziarka, M.~Wo{\l}czyk, J.~Tabor, J.~Nowak, and M.~{\'S}mieja.
\newblock Hypernetwork functional image representation.
\newblock In \emph{International Conference on Artificial Neural Networks}, pages 496--510. Springer, 2019.

\bibitem[Lee et~al.(2024)Lee, Kim, Li, Qiu, Hospedales, Husz{\'a}r, and Lane]{lee2024fedl2p}
R.~Lee, M.~Kim, D.~Li, X.~Qiu, T.~Hospedales, F.~Husz{\'a}r, and N.~Lane.
\newblock Fedl2p: Federated learning to personalize.
\newblock \emph{Advances in Neural Information Processing Systems}, 36, 2024.

\bibitem[Li et~al.(2020)Li, Sahu, Zaheer, Sanjabi, Talwalkar, and Smith]{li2020federated}
T.~Li, A.~K. Sahu, M.~Zaheer, M.~Sanjabi, A.~Talwalkar, and V.~Smith.
\newblock Federated optimization in heterogeneous networks.
\newblock \emph{Proceedings of Machine learning and systems}, 2:\penalty0 429--450, 2020.

\bibitem[Li et~al.(2024)Li, Wang, Chen, Lou, Chen, Yang, and Zheng]{li2024clients}
Y.~Li, T.~Wang, C.~Chen, J.~Lou, B.~Chen, L.~Yang, and Z.~Zheng.
\newblock Clients collaborate: Flexible differentially private federated learning with guaranteed improvement of utility-privacy trade-off.
\newblock \emph{arXiv preprint arXiv:2402.07002}, 2024.

\bibitem[Lin et~al.(2022)Lin, Han, Li, and Zhang]{lin2022personalized}
S.~Lin, Y.~Han, X.~Li, and Z.~Zhang.
\newblock Personalized federated learning towards communication efficiency, robustness and fairness.
\newblock \emph{Advances in Neural Information Processing Systems}, 35:\penalty0 30471--30485, 2022.

\bibitem[Ma et~al.(2022)Ma, Zhang, Guo, and Xu]{ma2022layer}
X.~Ma, J.~Zhang, S.~Guo, and W.~Xu.
\newblock Layer-wised model aggregation for personalized federated learning.
\newblock In \emph{Proceedings of the IEEE/CVF conference on computer vision and pattern recognition}, pages 10092--10101, 2022.

\bibitem[McAuley et~al.(2015)McAuley, Targett, Shi, and Van Den~Hengel]{mcauley2015image}
J.~McAuley, C.~Targett, Q.~Shi, and A.~Van Den~Hengel.
\newblock Image-based recommendations on styles and substitutes.
\newblock In \emph{Proceedings of the 38th international ACM SIGIR conference on research and development in information retrieval}, pages 43--52, 2015.

\bibitem[McMahan et~al.(2017)McMahan, Moore, Ramage, Hampson, and y~Arcas]{mcmahan2017communication}
B.~McMahan, E.~Moore, D.~Ramage, S.~Hampson, and B.~A. y~Arcas.
\newblock Communication-efficient learning of deep networks from decentralized data.
\newblock In \emph{Artificial intelligence and statistics}, pages 1273--1282. PMLR, 2017.

\bibitem[Navon et~al.(2020)Navon, Shamsian, Chechik, and Fetaya]{navon2020learning}
A.~Navon, A.~Shamsian, G.~Chechik, and E.~Fetaya.
\newblock Learning the pareto front with hypernetworks.
\newblock \emph{arXiv preprint arXiv:2010.04104}, 2020.

\bibitem[Panchal et~al.(2024)Panchal, Choudhary, Parikh, Zhang, and Guan]{panchal2024flow}
K.~Panchal, S.~Choudhary, N.~Parikh, L.~Zhang, and H.~Guan.
\newblock Flow: per-instance personalized federated learning.
\newblock \emph{Advances in Neural Information Processing Systems}, 36, 2024.

\bibitem[Paszke et~al.(2019)Paszke, Gross, Massa, Lerer, Bradbury, Chanan, Killeen, Lin, Gimelshein, Antiga, et~al.]{paszke2019pytorch}
A.~Paszke, S.~Gross, F.~Massa, A.~Lerer, J.~Bradbury, G.~Chanan, T.~Killeen, Z.~Lin, N.~Gimelshein, L.~Antiga, et~al.
\newblock Pytorch: An imperative style, high-performance deep learning library.
\newblock \emph{Advances in Neural Information Processing Systems}, 32, 2019.

\bibitem[Pillutla et~al.(2022)Pillutla, Malik, Mohamed, Rabbat, Sanjabi, and Xiao]{pillutla2022federated}
K.~Pillutla, K.~Malik, A.-R. Mohamed, M.~Rabbat, M.~Sanjabi, and L.~Xiao.
\newblock Federated learning with partial model personalization.
\newblock In \emph{International Conference on Machine Learning}, pages 17716--17758. PMLR, 2022.

\bibitem[Schneider and Vlachos(2021)]{schneider2021personalization}
J.~Schneider and M.~Vlachos.
\newblock Personalization of deep learning.
\newblock In \emph{Proceedings of International Data Science Conference}, pages 89--96. Springer, 2021.

\bibitem[Scott et~al.(2024)Scott, Zakerinia, and Lampert]{scott2024pefll}
J.~A. Scott, H.~Zakerinia, and C.~Lampert.
\newblock Pefll: Personalized federated learning by learning to learn.
\newblock In \emph{12th International Conference on Learning Representations}, 2024.

\bibitem[Sen et~al.(2008)Sen, Namata, Bilgic, Getoor, Galligher, and Eliassi-Rad]{sen2008collective}
P.~Sen, G.~Namata, M.~Bilgic, L.~Getoor, B.~Galligher, and T.~Eliassi-Rad.
\newblock Collective classification in network data.
\newblock \emph{AI magazine}, 29\penalty0 (3):\penalty0 93--93, 2008.

\bibitem[Shamsian et~al.(2021)Shamsian, Navon, Fetaya, and Chechik]{shamsian2021personalized}
A.~Shamsian, A.~Navon, E.~Fetaya, and G.~Chechik.
\newblock Personalized federated learning using hypernetworks.
\newblock In \emph{International Conference on Machine Learning}, pages 9489--9502. PMLR, 2021.

\bibitem[Shchur et~al.(2018)Shchur, Mumme, Bojchevski, and G{\"u}nnemann]{shchur2018pitfalls}
O.~Shchur, M.~Mumme, A.~Bojchevski, and S.~G{\"u}nnemann.
\newblock Pitfalls of graph neural network evaluation.
\newblock \emph{arXiv preprint arXiv:1811.05868}, 2018.

\bibitem[She et~al.(2023)She, Kang, Wang, Tay, Guan, Navarro, and Hartmannsgruber]{she2023image}
R.~She, Q.~Kang, S.~Wang, W.~P. Tay, Y.~L. Guan, D.~N. Navarro, and A.~Hartmannsgruber.
\newblock Image patch-matching with graph-based learning in street scenes.
\newblock \emph{IEEE Transactions on Image Processing}, 32:\penalty0 3465--3480, 2023.

\bibitem[Smith et~al.(2017)Smith, Chiang, Sanjabi, and Talwalkar]{smith2017federated}
V.~Smith, C.-K. Chiang, M.~Sanjabi, and A.~S. Talwalkar.
\newblock Federated multi-task learning.
\newblock \emph{Advances in Neural Information Processing Systems}, 30, 2017.

\bibitem[Tan et~al.(2023)Tan, Liu, Long, Jiang, Lu, and Zhang]{tan2023federated}
Y.~Tan, Y.~Liu, G.~Long, J.~Jiang, Q.~Lu, and C.~Zhang.
\newblock Federated learning on non-iid graphs via structural knowledge sharing.
\newblock In \emph{Proceedings of the AAAI conference on artificial intelligence}, pages 9953--9961, 2023.

\bibitem[Wan et~al.(2024)Wan, Huang, and Ye]{wan2024federated}
G.~Wan, W.~Huang, and M.~Ye.
\newblock Federated graph learning under domain shift with generalizable prototypes.
\newblock In \emph{Proceedings of the AAAI Conference on Artificial Intelligence}, pages 15429--15437, 2024.

\bibitem[Wang et~al.(2023)Wang, Fan, Qi, Jin, Yang, Shen, and Wang]{wang2023fedgs}
Z.~Wang, X.~Fan, J.~Qi, H.~Jin, P.~Yang, S.~Shen, and C.~Wang.
\newblock Fedgs: Federated graph-based sampling with arbitrary client availability.
\newblock In \emph{Proceedings of the AAAI Conference on Artificial Intelligence}, volume~37, pages 10271--10278, 2023.

\bibitem[Watts and Strogatz(1998)]{watts1998collective}
D.~J. Watts and S.~H. Strogatz.
\newblock Collective dynamics of ‘small-world’networks.
\newblock \emph{nature}, 393\penalty0 (6684):\penalty0 440--442, 1998.

\bibitem[Wu et~al.(2021)Wu, Wu, Cao, Huang, and Xie]{wu2021fedgnn}
C.~Wu, F.~Wu, Y.~Cao, Y.~Huang, and X.~Xie.
\newblock {FedGNN}: Federated graph neural network for privacy-preserving recommendation.
\newblock \emph{arXiv preprint arXiv:2102.04925}, 2021.

\bibitem[Xie et~al.(2021)Xie, Ma, Xiong, and Yang]{xie2021federated}
H.~Xie, J.~Ma, L.~Xiong, and C.~Yang.
\newblock Federated graph classification over non-iid graphs.
\newblock \emph{Advances in Neural Information Processing Systems}, 34:\penalty0 18839--18852, 2021.

\bibitem[Xu et~al.(2023)Xu, Yang, and Gu]{xu2023heterogeneous}
Z.~Xu, L.~Yang, and S.~Gu.
\newblock Heterogeneous federated learning based on graph hypernetwork.
\newblock In \emph{International Conference on Artificial Neural Networks}, pages 464--476. Springer, 2023.

\bibitem[Yan et~al.(2024)Yan, Wang, Sun, and Tong]{yan2024personalized}
Y.~Yan, S.~Wang, F.~Sun, and X.~Tong.
\newblock Personalized federated learning with multi-view geometry structure.
\newblock \emph{IEEE Internet of Things Journal}, 2024.

\bibitem[Ye et~al.(2023)Ye, Ni, Wu, Chen, and Wang]{ye2023personalized}
R.~Ye, Z.~Ni, F.~Wu, S.~Chen, and Y.~Wang.
\newblock Personalized federated learning with inferred collaboration graphs.
\newblock In \emph{International Conference on Machine Learning}, pages 39801--39817. PMLR, 2023.

\bibitem[Zhang et~al.(2018)Zhang, Ren, and Urtasun]{zhang2018graph}
C.~Zhang, M.~Ren, and R.~Urtasun.
\newblock Graph hypernetworks for neural architecture search.
\newblock \emph{arXiv preprint arXiv:1810.05749}, 2018.

\bibitem[Zhang et~al.(2021{\natexlab{a}})Zhang, Guo, Ma, Wang, Xu, and Wu]{zhang2021parameterized}
J.~Zhang, S.~Guo, X.~Ma, H.~Wang, W.~Xu, and F.~Wu.
\newblock Parameterized knowledge transfer for personalized federated learning.
\newblock \emph{Advances in Neural Information Processing Systems}, 34:\penalty0 10092--10104, 2021{\natexlab{a}}.

\bibitem[Zhang et~al.(2021{\natexlab{b}})Zhang, Yang, Li, Sun, and Yiu]{zhang2021subgraph}
K.~Zhang, C.~Yang, X.~Li, L.~Sun, and S.~M. Yiu.
\newblock Subgraph federated learning with missing neighbor generation.
\newblock \emph{Advances in Neural Information Processing Systems}, 34:\penalty0 6671--6682, 2021{\natexlab{b}}.

\bibitem[Zhao et~al.(2024)Zhao, Li, Kang, Ji, Ding, Zhao, Liang, and Tay]{zhao2024distributed}
K.~Zhao, X.~Li, Q.~Kang, F.~Ji, Q.~Ding, Y.~Zhao, W.~Liang, and W.~P. Tay.
\newblock Distributed-order fractional graph operating network.
\newblock \emph{Advances in Neural Information Processing Systems}, 37:\penalty0 103442--103475, 2024.

\end{thebibliography}

\newpage
\onecolumn
\appendix
\renewcommand{\thetheorem}{\arabic{theorem}}
\setcounter{theorem}{0}
\section{Appendix - Experimental Setups}\label[Appendix]{app_setup}

\subsection{Datasets}\label[Appendix]{app_dataset}
\label{app_datasets}
\vspace{-3mm}
\begin{table*}[!ht]
\caption{Dataset statistics for the non-overlapping node scenario for both the original graph and its split subgraphs \protect\cite{baek2023personalized}. ``Ori'' denotes the original largest connected components in the graph. The clustering coefficient, reflecting node clustering within subgraphs, is calculated by averaging node coefficients \protect\cite{watts1998collective}. Heterogeneity, indicating subgraph dissimilarity, is measured by the median Jensen-Shannon divergence of label distributions across subgraph pairs.}
\label{nonoverlapping_dataset_statistics} 
\vspace{-2mm}
\begin{center}
\footnotesize
\resizebox{\textwidth}{!}{
\begin{tabular}{lccccccccc}
\toprule
\multirow{2}{*}{} & \multicolumn{3}{c}{Cora}   & \multicolumn{3}{c}{Citeseer}   & \multicolumn{3}{c}{Pubmed}      \\ \cline{2-10} 
                         & Ori & 10 Clients & 20 Clients  & Ori & 10 Clients & 20 Clients  & Ori & 10 Clients & 20 Clients    \\ 
\midrule
\# Classes               & \multicolumn{3}{c}{7}       & \multicolumn{3}{c}{6}       & \multicolumn{3}{c}{3}               \\
\#  Nodes                & 2,485   &  249    &  124    &  2,120   &  212   & 106     &  19,717   &  1,972   &  986         \\
\# Edges                 & 10,138  &  891    &  422    &  7,358   &  675   & 326     &  88,648   &  7,671   & 3,607        \\
Clustering Coefficient   & 0.238   &  0.259  &  0.263  &  0.170   &  0.178 & 0.180   &  0.060    &  0.066   & 0.067        \\
Heterogeneity            & N/A     &  0.606  &  0.665  &  N/A     &  0.541 & 0.568   &  N/A      &  0.392   & 0.424        \\ 
\bottomrule
\multicolumn{10}{l}{}                                                                                                                         \\ 
\toprule
\multirow{2}{*}{} & \multicolumn{3}{c}{Amazon-Computer} & \multicolumn{3}{c}{Amazon-Photo}    & \multicolumn{3}{c}{ogbn-arxiv}        \\ \cline{2-10} 
                & Ori & 10 Clients & 20 Clients  & Ori & 10 Clients & 20 Clients  & Ori & 10 Clients & 20 Clients    \\ 
\midrule
\# Classes               & \multicolumn{3}{c}{10}       & \multicolumn{3}{c}{8}       & \multicolumn{3}{c}{40}               \\
\#  Nodes                &  13,381    &  1,338          &   669          & 7,487      &   749         &      374      &  169,343    &   16,934         & 8,467   \\
\# Edges                 &  491,556   &    36,136       &     15,632     & 238,086    &    19,322     &     8,547     &  2,315,598  &    182,226       &  86,755\\
Clustering Coefficient   &   0.351    &    0.398        &   0.418        &  0.410     &    0.457      &        0.477  &  0.226      &  0.259           &   0.269\\
Heterogeneity            &  N/A       &   0.612         &    0.647       &  N/A       &     0.681     &    0.751      & N/A         &    0.615         &  0.637        \\ 
\bottomrule
\end{tabular}
}
\vspace{-7mm}
\end{center}
\end{table*}

\begin{table*}[!ht]
\caption{Dataset statistics for the overlapping node scenario for both the original graph and its split subgraphs.}
\label{overlapping_dataset_statistics}
\vspace{-2mm}
\begin{center}
\footnotesize
\resizebox{\textwidth}{!}{
\begin{tabular}{lccccccccc}
\toprule
\multirow{2}{*}{} & \multicolumn{3}{c}{Cora}   & \multicolumn{3}{c}{Citeseer}   & \multicolumn{3}{c}{Pubmed}  \\ \cline{2-10} 
                & Ori & 30 Clients & 50 Clients  & Ori & 30 Clients & 50 Clients  & Ori & 30 Clients & 50 Clients   \\ 
\midrule
\# Classes               & \multicolumn{3}{c}{7}       & \multicolumn{3}{c}{6}       & \multicolumn{3}{c}{3}               \\
\#  Nodes                & 2,485    &   207         &  124           &  2,120   &     177      &     106        
&   19,717   &    1,643        &     986          \\
\# Edges                 & 10,138   &   379         &    215         &  7,358   &   293       &     170        
&   88,648  &    3,374        &      1,903 \\
Clustering Coefficient   & 0.238    &  0.129          &   0.125           &  0.170   &    0.087      &   0.096          
&  0.060    &     0.034        &     0.035          \\
Heterogeneity            & N/A      &   0.567         &  0.613           &  N/A     &   0.494        &   0.547          
&  N/A   &     0.383        &    0.394   \\ 
\bottomrule
\multicolumn{10}{l}{}                                                                                                                         \\ 
\toprule
\multirow{2}{*}{} & \multicolumn{3}{c}{Amazon-Computer} & \multicolumn{3}{c}{Amazon-Photo}    & \multicolumn{3}{c}{ogbn-arxiv}        \\ \cline{2-10} 
                & Ori & 30 Clients & 50 Clients & Ori & 30 Clients & 50 Clients & Ori & 30 Clients & 50 Clients  \\ 
\midrule
\# Classes               & \multicolumn{3}{c}{10}       & \multicolumn{3}{c}{8}       & \multicolumn{3}{c}{40}               \\
\#  Nodes                &  13,381    &    1,115         &   669          & 7,487      &     624       &       374     &  169,343    &   14,112         &    8,467   \\
\# Edges                 &  491,556   &   16,684         &   8,969        & 238,086    &     8,735     &      4,840    &  2,315,598  &     83,770       &   44,712 \\
Clustering Coefficient   &   0.351    &    0.348         &   0.359        &  0.410     &    0.391      &      0.410    &  0.226      &  0.185           &   0.191      \\
Heterogeneity            &  N/A       &   0.577          &     0.614      &  N/A       &     0.696     &       0.684   & N/A         &    0.606         &   0.615       \\ 
\bottomrule
\end{tabular}
}
\end{center}
\vspace{-3mm}
\end{table*}

We present statistical analyses derived from six distinct benchmark datasets \cite{sen2008collective,mcauley2015image,shchur2018pitfalls,hu2020open}: Cora, CiteSeer, Pubmed, and ogbn-arxiv for citation graphs, as well as Computer and Photo for Amazon product graphs. These datasets serve as the foundation for our experimental investigations, covering both non-overlapping and overlapping and node scenarios, detailed in Tables~\ref{nonoverlapping_dataset_statistics} and \ref{overlapping_dataset_statistics}. The table provides a comprehensive overview of key metrics for each subgraph, encompassing node and edge counts, class distribution, and clustering coefficients \cite{baek2023personalized}. The clustering coefficient, indicating the extent of node clustering within individual subgraphs, is calculated by first determining the coefficient \cite{watts1998collective} for all nodes and subsequently computing the average. The heterogeneity, reflecting dissimilarity among disjointed subgraphs, is measured by calculating the median Jenson-Shannon divergence of label distributions across all subgraph pairs. These statistics offer a detailed understanding of the structural intricacies and relationships within the benchmark datasets utilized in the experiments. In partitioning datasets, we randomly allocate 40\% of nodes for training, 30\% for validation, and another 30\% for testing across all datasets.

We delineate the procedure for partitioning the original graph into multiple subgraphs, tailored to the number of clients participating in FL. The METIS graph partitioning algorithm \cite{karypis1997metis} is employed to effectively segment the original graph into distinct subgraphs, offering control over the number of disjoint subgraphs through adjustable parameters. In scenarios with non-overlapping nodes, each client is assigned a unique disjoint subgraph directly derived from the METIS algorithm output. For example, setting the METIS parameter to 10 results in 10 distinct disjoint subgraphs, each allocated to an individual client.
Conversely, in scenarios involving overlapping nodes across subgraphs, the process begins by dividing the original graph into 6 and 10 disjoint subgraphs for configurations with 30 and 50 clients, respectively, utilizing the METIS algorithm. Subsequently, within each split subgraph, half of the nodes and their associated edges are randomly sampled, forming a subgraph assigned to a specific client. This iterative process is repeated 5 times, generating 5 distinct yet overlapped subgraphs for each split subgraph obtained from METIS. This meticulous approach ensures a varied yet controlled distribution of data across clients, accommodating both non-overlapping and overlapping node scenarios within the framework of FL.

\subsection{Algorithm}
\label[Appendix]{app_algorithm}

\begin{algorithm}[H]
\caption{FedSheafHN}
\label{alg_fedsheafhn}
\begin{algorithmic}
    \State {\bfseries Input:} Communication rounds $R$, number of clients $N$, number of client local training rounds $T_c$, interval between client collaboration graph update $r_{in}$.
    \State {\bfseries Initialization:} 
    \For{client $i\in [N]$}
        \State Initialize client model parameter $\bomega_{i}=(\bomega_{b,i}, \bomega_{h,i})$. 
        \State \textproc{Client-Update}$(0, i, \bomega_{b,i})$
    \EndFor
    \State \
    \State {\bfseries Server:}
    \For{communication round $r\in \set{0,\dots,R}$}
        \If{$(r \mod r_{in}) == 0$}
            \State Set $\bX_0=(\bx_i\tc{c})_{i\in [N]}$ received from clients.
            \State Update $G_{s}$ in (\ref{Gs}) with received $\bX_0$.
        \Else
            \State Set the current $\bX_0$ to the value of $\bX_0$ from the previous communication round.
        \EndIf 
        \State Compute $\bX_{T_s}=S(\bX_0;\btheta)$. 
        \State Generate $\bOmega_b\tc{s} = H(\bX_{T_s};\bm{\varphi})$ and transmit to clients.
        \For{client $i$ in $[N]$}
            \State \textproc{Client-Update}$(r, i, \bomega_{b,i}\tc{s})$ 
        \EndFor
        \State Gradient descent update of $\btheta$ via $\btheta := \btheta - \alpha  \Delta \btheta$. 
        \State Gradient descent update of $\bm{\varphi}$ via $\bm{\varphi} := \bm{\varphi} - \beta \Delta \bm{\varphi}$.
    \EndFor
    \State \
    \State {\bfseries Client:}
    \Procedure{Client-Update}{$r$, $i$, $\bomega_{b,i}\tc{s}$} 
        \State Set $\bomega_{b,i}=\bomega_{b,i}\tc{s}$. 
        \For{epoch $\in [T_c]$}
            \State Gradient descent update of $\bomega_{i}$ via $\bomega_i := \bomega_i - \gamma \nabla_{\bomega_i}\mathcal{L}_i(\bomega_i)$.
        \EndFor
        \If{$(r \mod r_{in}) == 0$}
            \State Compute graph-level embedding $\bx_i\tc{c}$ using (\ref{graph_embedding}) and transmit to the server.
        \EndIf
        \State Transmit $\Delta \bomega_{b,i}:= \bomega_{b,i}^{(c)}-\bomega_{b,i}^{(s)}$ to the server.
    \EndProcedure
\end{algorithmic}
\end{algorithm}

\subsection{Hyperparameters}\label[Appendix]{app_hyperparameters}
The hyperparameter ranges used in our experiments are listed in Table~\ref{hp_range}.

\begin{table}[H]
\caption{Hyperparameter ranges}
\label{hp_range}
\vspace{3mm}
\centering
\newcommand{\tabincell}[2]{\begin{tabular}{@{}#1@{}}#2\end{tabular}} 
\begin{tabular}{lcc}
\toprule
\multirow{4}{*}{\tabincell{c}{Sheaf diffusion \\ model}} & Learning rate   & \tabincell{c}{(0.02, 0.01, 0.001, \\ 0.0001, 0.00001) \\}\\
\multirow{4}{*}{}                      & Hidden channels & (5, 10, 20) \\
\multirow{4}{*}{}                      & \tabincell{c}{Weight decay  \\} & 0.0005 \\
\multirow{4}{*}{}                      & Layer dropout & 0-0.6 \\
\midrule
\multirow{2}{*}{Hypernetworks}         & Learning rate & (0.02, 0.01, 0.001 )\\
\multirow{2}{*}{}                      & Layer dropout & 0.3 \\
\midrule
\multirow{4}{*}{Client model}      & Learning rate & (0.02, 0.01, 0.001)\\
\multirow{4}{*}{}                      & Weight decay & 0.0005 \\
\multirow{4}{*}{}                      & Layer dropout  & 0-0.6 \\
\bottomrule
\end{tabular}
\end{table}

\newpage
\section{Appendix - Convergence Proofs} \label[Appendix]{app_convergence-proofs}

Let $\omega_t=(\eta_t,\omega_{h,i,t,i\in[N]})$ denotes the parameters at communication round t.
$g_{t,i,l}$ denotes stochastic gradient for client $i$ at round $t$, local epoch $l$. Denote $\omega_{b,i}=M(\eta)$.

The overall objective is $\mathcal{L}(\omega)=\frac{1}{N} \sum_{i=1}^{N}\mathcal{L}_i(\omega_i)$, where
\begin{align*}
    \mathcal{L}_i(G,\omega_i)&=\mathcal{L}_i(G,(\omega_{b,i},\omega_{h,i})\\
    &=\mathcal{L}_i(G,(f_{i,1}(\omega_{b,i})),f_{i,2}(\omega_{h,i})
\end{align*}

\subsection{Assumptions}\label{assumptions}

Following the literature \cite{scott2024pefll}, we make the assumptions separately for each component.

\begin{itemize}
\item Bounded Gradients: there exist constants $b_1, b_2, b_3$, such that
\begin{align*}
&\|\nabla_{\omega_{h,i}} f_{i,1}(\omega_{b,i})\| \leq b_1  \\
&\|\nabla_{\omega_{h,i}} f_{i,1}(\omega_{h,i})\| \leq b_2  \\
&\|\nabla_{y} \omega_{b,i}\| \leq b_3 
\end{align*}

\item Smoothness: there exist constants $L_1, L_4, L_5, L_7$, such that
\begin{align*}
&\|\omega_{b,i} - \omega_{b,i}^\prime\| \leq L_1 \|y - y^\prime\|  \\
&\|\nabla_{\omega_{h,i}} f_{i,2}(\omega_{h,i}) - \nabla_{\omega_{h,i}} f_{i,2}(\omega_{h,i}^\prime)\| 
\leq L_4 \|\omega_{h,i} - \omega_{h,i}^\prime\|  \\
&\|\nabla_{\omega_{b,i}} f_{i,1}(\omega_{b,i}) - \nabla_{\omega_{b,i}} f_{i,1}(\omega_{b,i}^\prime)\| 
\leq L_7 \|\omega_{b,i} - \omega_{b,i}^\prime\|  \\
&\|\nabla_{y} \omega_{b,i} - \nabla_{y} \omega_{b,i}^\prime\| 
\leq L_5 \|\eta - \eta^\prime\| 
\end{align*}

\item Bounded Variance: there exists a constant $\sigma_1$, such that
\begin{equation*}
\mathbb{E}\|\nabla_{\omega_{b,i}} f_{i,1}(\omega_{b,i}) - \nabla_{\omega_{b,i}} \hat{f_{i,1}}(\omega_{b,i})\|^2 
\leq \sigma_1^2 
\end{equation*}

\item Bounded Dissimilarity: there exists a constant $L_G$, such that with $\overline{\mathcal{L}}(\omega) = \frac{1}{N} \sum_{i=1}^N \mathcal{L}_i(\omega)$:
\begin{align*}
\frac{1}{N} \sum_{i=1}^N \mathbb{E} \|\nabla \mathcal{L}_i(\omega) - \nabla \overline{\mathcal{L}}(\omega)\|^2 &\leq L_G^2 
\end{align*}

\end{itemize}

\subsection{Convergence}
\begin{Lemma}\label{Lemma1}
    There exists a constant $L$ which $\nabla \mathcal{L}(\omega)$ is $L$-smooth.
\end{Lemma}

\begin{proof}
For each $\mathcal{L}_i$ we have

\begin{align*}
\|\nabla_\omega \mathcal{L}_i(\omega_i) - \nabla_\omega \mathcal{L}_i(\omega_i^\prime)\| &= \|\nabla_\omega \mathcal{L}_i(\omega_{b,i},\omega_{h,i})- \nabla_\omega \mathcal{L}_i(\omega_{b,i}^\prime,\omega_{h,i}^\prime)\| \\
&= \|\nabla_{\omega_{b,i}} f_{i,1}(\omega_{b,i})^\top \nabla_\eta \omega_{b,i}^\top
- \nabla_{\omega_{b,i}} f_{i,1}(\omega_{b,i}^\prime)^\top \nabla_\eta \omega_{b,i}^{\prime\top}\| \\
&\quad + \|\nabla_{\omega_{h,i}} f_{i,2}(\omega_{h,i})^\top \nabla_\eta \omega_{h,i}^\top
- \nabla_{\omega_{h,i}} f_{i,2}(\omega_{h,i}^\prime)^\top \nabla_\eta \omega_{h,i}^{\prime\top}\| \\
&= \Big\|  \nabla_{\omega_{b,i}} f_{i,1}(\omega_{b,i})^\top \nabla_\eta \omega_{b,i}^\top
-  \nabla_{\omega_{b,i}} f_{i,1}(\omega_{b,i})^\top \nabla_\eta \omega_{b,i}^{\prime\top} \notag \\
&\quad + \nabla_{\omega_{b,i}} f_{i,1}(\omega_{b,i})^\top \nabla_\eta \omega_{b,i}^{\prime\top} 
- \nabla_{\omega_{b,i}}\mathcal{L}_i(\omega_i^\prime)\nabla_{\omega_{b,i}} f_{i,1}(\omega_{b,i}^\prime)^\top \nabla_\eta \omega_{b,i}^{\prime\top} \Big\| \\
&\quad + \|\nabla_{\omega_{h,i}} f_{i,2}(\omega_{h,i})^\top \nabla_\eta \omega_{h,i}^\top
- \nabla_{\omega_{h,i}} f_{i,2}(\omega_{h,i}^\prime)^\top \nabla_\eta \omega_{h,i}^{\prime\top}\| \\
&\leq b_2 L_5 \|\eta - \eta^\prime\| + b_3 L_7 \|\omega_{b,i} - \omega_{b,i}^\prime\|
+ L_4 \|\omega_{h,i} - \omega_{h,i}^\prime\| \\
&= (b_2 L_5 + b_3 L_7) \|\eta - \eta^\prime\|
+ L_4 \|\omega_{h,i} - \omega_{h,i}^\prime\| 
\end{align*}
Let $\omega_i=(\eta, \omega_{h,i})$, $L=\max\{b_2 L_5 + b_3 L_1 L_7, L_4\}$
\begin{align*}
   \|\nabla_\omega \mathcal{L}_i(\omega_i) - \nabla_\omega \mathcal{L}_i(\omega_i^\prime)\| \leq L \|\omega_i - \omega_i^\prime \|
\end{align*}
Therefore for $\mathcal{L}$ we have
\begin{align*}
\|\nabla_\omega \mathcal{L}(\omega) - \nabla_\omega{\mathcal{L}}(\omega^\prime)\|  
&\leq \bigg\| \frac{1}{N} \sum_{i=1}^N \nabla_\omega \mathcal{L}_i(\omega) - \frac{1}{N} \sum_{i=1}^N \nabla_\omega \mathcal{L}_i(\omega^\prime) \bigg\|   \\
&\leq L_2  \|\eta - \eta^\prime\|  + \frac{1}{N} \sum_{i=1}^N L_2  \|\omega_{h,i} - \omega_{h,i}^\prime\|   \\
&= L_2  \|\omega - \omega^\prime\| 
\end{align*}
\end{proof}

\begin{Lemma}\label{Lemma2}
For any step $t$ and client $i$ the following inequality holds:
\begin{equation*}
\mathbb{E}_t \|\nabla \mathcal{L}_i(\omega_t) - \nabla_{\omega_{b,i}} f_{i,1}(\omega_{b,i}^t) 
- \beta \sum_{j=0}^{l-1} g_{t,i,j}(\omega_{b,i}^t) \nabla_{\eta_t}\omega_{b,i}^{t\top}\|^2 
\leq 3 L_7^2 b_3^2 \beta^2 l^2 (b_2^2 + \sigma_1^2).
\end{equation*}
\end{Lemma}

\begin{proof}
Through the addition and subtraction of intermediate terms, followed by the application of the Cauchy–Schwarz inequality, we obtain
\begin{align*}
&\mathbb{E}_t \|\nabla \mathcal{L}_i(\omega_t) - \nabla_{\omega_{b,i}} f_{i,1}(\omega_{b,i}^t) 
- \beta \sum_{j=0}^{l-1} g_{t,i,j}(\omega_{b,i}^t) \nabla_{\eta_t}\omega_{b,i}^{t\top}\|^2 \\
&\leq 3 \mathbb{E} \|\nabla_{\omega_{b,i}} f_{i,1}(\omega_{b,i}^t) 
- \nabla_{\omega_{b,i}} f_{i,1}(\omega_{b,i}^t - \beta \sum_{j=0}^{l-1} g_{t,i,j}(\omega_{b,i}^t))\|^2 \|\nabla_{\eta_t}\omega_{b,i}^t\|^2   \\
&\leq 3 L_7^2 b_3^2 \beta^2 \mathbb{E} \Big\|\sum_{j=0}^{l-1} g_{t,i,j}(\omega_{b,i}^t)\Big\|^2  \\
&\leq 3 L_7^2 b_3^2 \beta^2 l^2 (b_2^2 + \sigma_1^2) 
\end{align*}
\end{proof}

\begin{Lemma}\label{Lemma3}
For any step $t$ the following inequality holds:
\begin{align*}
    &\mathbb{E} \langle \nabla \mathcal{L}(\omega_t), \omega_{t+1} - \omega_t \rangle \leq -\frac{3 \beta T_c}{4} \mathbb{E} \|\nabla \mathcal{L}(\omega_t)\|^2 + L_7^2 b_3^2 \beta^3 T_c^3 (b_2^2 + \sigma_1^2) + \beta b_1^2
\end{align*}
\end{Lemma}

\begin{proof}
\begin{align*}
&\mathbb{E} \langle \nabla \mathcal{L}(\omega_t), \omega_{t+1} - \omega_t \rangle \leq \langle \nabla_\eta \mathcal{L}(\omega_t),\eta_{t+1}-\eta_t \rangle + \frac{1}{N} \sum_{i=1}^{N} \langle \nabla_{\omega_{h,i}}\mathcal{L}(\omega_t),\omega_{h,i,t+1},\omega_{h,i,t} \rangle \\
&= \mathbb{E} \bigg[ \langle \nabla_\eta \mathcal{L}(\omega_t), -\frac{\beta}{N} \sum_{i=1}^N \sum_{l=1}^{T_c} g_{t,i,l}(\omega_{b,i,t}) \nabla_{\eta_t} \omega_{b,i,t}^\top \rangle \bigg] + \frac{1}{N} \sum_{i=1}^N \mathbb{E} \langle \nabla_{\omega_{h,i}} \mathcal{L}(\omega_t), \omega_{h,i,t+1} - \omega_{h,i,t} \rangle \\
&=\mathbb{E}_t \bigg[ \langle \nabla_\eta \mathcal{L}(\omega_t), -\frac{\beta}{N} \sum_{i=1}^N \sum_{l=1}^{T_c} \nabla_{\omega_{b,i}}f_{i,1} \bigg(\omega_{b,i,t} - \beta \sum_{j=0}^{l-1} g_{t,i,j}(\omega_{b,i,t}) \bigg)\nabla_{\eta_t} \omega_{b,i,t}^\top \rangle \bigg] \\
&\quad + \frac{1}{N} \sum_{i=1}^N \mathbb{E} \langle \nabla_{\omega_{h,i}} \mathcal{L}(\omega_t), \omega_{h,i,t+1} - \omega_{h,i,t} \rangle\\ 
&= \frac{\beta}{N} \sum_{i=1}^N \sum_{l=1}^{T_c} \mathbb{E}_t \bigg[ \langle \nabla_\eta \mathcal{L}(\omega_t), \nabla_\eta \mathcal{L}_i(\omega_t) - \nabla_{\omega_{b,i}}f_{i,1} \bigg(\omega_{b,i,t} - \beta \sum_{j=0}^{l-1} g_{t,i,j}(\omega_{b,i,t}) \bigg)\nabla_{\eta_t} \omega_{b,i,t}^\top \rangle \bigg] \\
&\quad + \frac{\beta}{N} \sum_{i=1}^N \sum_{l=1}^{T_c} \mathbb{E}_t \bigg[\langle \nabla \mathcal{L}(\omega_t), \nabla \mathcal{L}_i(\omega_t) \rangle \bigg] + \frac{1}{N} \sum_{i=1}^N \mathbb{E} \langle \nabla_{\omega_{h,i}} \mathcal{L}(\omega_t),\omega_{h,i,t+1}-\omega_{h,i,t}\rangle 
\end{align*}

By definition $\mathcal{L}_\omega = \frac{1}{N} \sum_{i=1}^{N} \mathcal{L}_i(\omega_i)$, and by Young’s inequality and Lemma~\ref{Lemma2} we have

\begin{align*}
&\mathbb{E} \langle \nabla \mathcal{L}(\omega_t), \omega_{t+1} - \omega_t \rangle \\
&\leq \frac{\beta T_c}{4} \|\nabla_\eta \mathcal{L}(\omega_t)\|^2 + \frac{\beta}{N} \sum_{i=1}^N \sum_{l=1}^{T_c} \mathbb{E}_t \Big\|\nabla_\eta \mathcal{L}_i(\omega_t) 
- \nabla_{\omega_{b,i}} f_{i,l}(\omega_{b,i,t}) - \beta \sum_{j=0}^{l-1} g_{t,i,j}(\omega_{b,i,t}) \nabla_{\eta_t} \omega_{b,i,t}^\top \Big\|^2 \notag \\
&- \beta T_c \|\nabla_\eta \mathcal{L}(\omega_t)\|^2 
+ \frac{1}{N} \sum_{i=1}^N \mathbb{E} \langle \nabla_{\omega_{h,i}} \mathcal{L}(\omega_t), \omega_{h,i,t+1} - \omega_{h,i,t} \rangle \\
&\leq -\frac{3 \beta T_c}{4} \mathbb{E} \|\nabla \mathcal{L}(\omega_t)\|^2 + \frac{\beta}{N} \sum_{i=1}^N \sum_{l=1}^{T_c} \Big( 3 L_7^2 b_3^2 \beta^2 l^3 (b_2^2 + \sigma_1^2) \Big) + \beta b_1^2  \\
&=-\frac{3 \beta T_c}{4} \mathbb{E} \|\nabla \mathcal{L}(\omega_t)\|^2 + L_7^2 b_3^2 \beta^3 T_c^3 (b_2^2 + \sigma_1^2) + \beta b_1^2
\end{align*}
\end{proof}

\begin{Lemma}\label{Lemma4}
For any step $t$ the following inequality holds:
\begin{align*}
    \mathbb{E} [ \| \eta_{t+1} - \eta_t \|^2 ] \leq \frac{4\beta^2 T_c}{N} \sigma_1^2 
+ 12\beta^4 L_7^2 b_3^2 T_c^4 (b_2^2 + \sigma_1^2) 
+ \frac{4\beta^2 T_c^2}{N} L_G^2 
+ 4\beta^2 T_c^2 \mathbb{E} \bigg[ \|\nabla \mathcal{L}(\omega_t)\|^2 \bigg].
\end{align*}
\end{Lemma}

\begin{proof}
\begin{align*}
\mathbb{E} [ \| \eta_{t+1} - \eta_t \|^2 ]
&= \mathbb{E} \bigg[\bigg\| - \frac{\beta}{N} \sum_{i=1}^N \sum_{l=1}^{T_c} g_{t,i,l}(\omega_{b,i}^t) \nabla_{\eta_t} \omega_{b,i}^t \bigg\|^2 \bigg] \\
&= \frac{\beta^2}{N^2} \mathbb{E} \bigg[\bigg\| \sum_{i=1}^N \sum_{l=1}^{T_c} g_{t,i,l}(\omega_{b,i}^t) \nabla_{\eta_t} \omega_{b,i}^t \bigg\|^2 \bigg] \\
&= \frac{\beta^2}{N^2} \mathbb{E} \bigg[ \bigg\| \sum_{i=1}^N \sum_{l=1}^{T_c} 
\Big( g_{t,i,l}(\omega_{b,i}^t) \nabla_{\eta_t} \omega_{b,i}^t - \nabla_{\omega_{b,i}} f_{i,l}\big(\omega_{b,i}^t - \beta \sum_{j=0}^{l-1} g_{t,i,j}(\omega_{b,i}^t)\big)\nabla_{\eta_t}{\omega_{b,i}^t}^\top \Big) \notag \\
&\quad - \sum_{i=1}^N \sum_{l=1}^{T_c}\nabla_{\omega_{b,i}} f_{i,l}\big(\omega_{b,i}^t - \beta \sum_{j=0}^{l-1} g_{t,i,j}(\omega_{b,i}^t)\big) - \nabla \mathcal{L}_i(\omega_t) \notag \\
&\quad + \sum_{i=1}^N \sum_{l=1}^{T_c} \Big( \nabla \mathcal{L}_i(\omega_t) - \nabla \mathcal{L}_i(\omega_t) \Big) + \sum_{i=1}^N \sum_{l=1}^{T_c} \nabla \mathcal{L}(\omega_t)\bigg\|^2 \bigg]
\end{align*}
By Cauchy–Schwarz inequality we get
\begin{align*}
\mathbb{E} [ \| \eta_{t+1} - \eta_t \|^2 ]
&\leq \frac{4\beta^2}{N^2} \mathbb{E} \bigg[\bigg\| \sum_{i=1}^N \sum_{l=1}^{T_c} g_{t,i,l}(\omega_{b,i}^t) \nabla_{\eta_t} \omega_{b,i}^t - \nabla_{\omega_{b,i}} f_{i,l}\big(\omega_{b,i}^t - \beta \sum_{j=0}^{l-1} g_{t,i,j}(\omega_{b,i}^t)\big){\nabla_{\eta_t}\omega_{b,i}^t}^\top \bigg\|^2 \bigg] \notag \\
&\quad + \frac{4 \beta^2}{N^2} \mathbb{E} \bigg[ \Big\| \sum_{i=1}^N \sum_{l=1}^{T_c} 
\Big( \nabla_{\omega_{b,i}} f_{i,l}(\omega_{b,i}^t) - \beta \sum_{j=0}^{l-1} g_{t,i,j}(\omega_{b,i}^t) \nabla_{\eta_t} \omega_{b,i}^t 
- \nabla \mathcal{L}_i(\omega_t) \Big) \Big\|^2 \bigg] \notag \\
&\quad + \frac{4 \beta^2}{N^2} \mathbb{E} \bigg[ \Big\| \sum_{i=1}^N \sum_{l=1}^{T_c} 
\Big( \nabla \mathcal{L}_i(\omega_t) - \nabla \mathcal{L}(\omega_t) \Big) \Big\|^2 \bigg]  + \frac{4 \beta^2}{N^2} \mathbb{E} \bigg[ \Big\| \sum_{i=1}^N \sum_{l=1}^{T_c} \nabla \mathcal{L}(\omega_t) \Big\|^2 \bigg] \\
&\leq \frac{4 \beta^2}{N^2} \sum_{i=1}^N \sum_{l=1}^{T_c} 
\mathbb{E} \bigg[ \Big\| g_{t,i,l}(\omega_{b,i}^t) \nabla_{\eta_t} \omega_{b,i}^t 
- \nabla_{\omega_{b,i}} f_{i,l}\big(\omega_{b,i}^t - \beta \sum_{j=0}^{l-1} g_{t,i,j}(\omega_{b,i}^t) \big) {\nabla_{\eta_t} \omega_{b,i}^t}^{\top} \Big\|^2 \bigg] \\
\end{align*}
By Lemma~\ref{Lemma2}, the bounded dissimilarity assumption and simplifying terms, we obtain
\begin{align*}
\mathbb{E} [ \| \eta_{t+1} - \eta_t \|^2 ]
&\leq \frac{4\beta^2 T_c}{N} \sigma_1^2 + 4\beta^4 \cdot 3 L_7^2 b_3^2 T_c^4 (b_2^2 + \sigma_1^2) \notag \\
&\quad + \frac{4\beta^2 T_c^2}{N^2} \sum_{i=1}^N \mathbb{E} \bigg[ \|\nabla \mathcal{L}_i(\omega_t) - \frac{1}{N} \sum_{i=1}^N \nabla \mathcal{L}_i(\omega_t)\|^2 \bigg] 
+ 4\beta^2 T_c^2 \mathbb{E} \bigg[ \|\nabla \mathcal{L}(\omega_t)\|^2 \bigg] \notag \\
&\leq \frac{4\beta^2 T_c}{N} \sigma_1^2 
+ 12\beta^4 L_7^2 b_3^2 T_c^4 (b_2^2 + \sigma_1^2) 
+ \frac{4\beta^2 T_c^2}{N} L_G^2 
+ 4\beta^2 T_c^2 \mathbb{E} \bigg[ \|\nabla \mathcal{L}(\omega_t)\|^2 \bigg].
\end{align*}
\end{proof}


\begin{theorem}
Under standard assumptions of smoothness and boundedness, the optimization of FedSheafHN after $T$ communication round satisfies
\begin{align*}
\frac{1}{T}\sum_{t=1}^{T}\mathbb{E}\big\|\nabla \mathcal{L}(\omega_t)\big\|^2\le \frac{(\mathcal{L}(\omega_0)-\mathcal{L}^{\ast})}{\sqrt{NT}}+\frac{LL_1(16\sigma_1^2+8T_cL_G^2)}{T_c\sqrt{NT}} 
+\frac{16NL_7^2b_3^2b_2^2}{T}+\frac{(L+2)}{T_c}b_1^2 
\end{align*}
\end{theorem}

\begin{proof}
By smoothness of $\mathcal{L}$ we have
\begin{align*}
\mathbb{E}[\mathcal{L}(\omega_{t+1}) - \mathcal{L}(\omega_t)] 
&\leq \mathbb{E} \big[ \big\langle \nabla \mathcal{L}(\omega_t), \omega_{t+1} - \omega_t \big\rangle \big] + \frac{L}{2} \mathbb{E} \bigg[ \big\| \omega_{b,t+1}-\omega_{b,t} \big\|^2 
+ \frac{1}{N} \sum_{i=1}^N \big\| \omega_{h,i,t+1} - \omega_{h,i,t} \big\|^2 \bigg]
\end{align*}
And by combining it with Lemma~\ref{Lemma3} and Lemma~\ref{Lemma4} and simplifying terms we get
\begin{align*}
&\mathbb{E}[\mathcal{L}(\omega_{t+1}) - \mathcal{L}(\omega_t)]
\leq -\frac{3\beta T_c}{4} \mathbb{E} \big\| \nabla \mathcal{L}(\omega_t) \big\|^2 
+ L_7^2 b_3^2 \beta^3 T_c^3 \big(b_2^2 + \sigma_1^2\big) 
+ \beta b_1^2 \notag \\
&\quad + \frac{2\beta^2 T_c L L_1}{N} \sigma_1^2 
+ 6L \beta^4 L_7^2 b_3^2 T_c^4 L_1\big(b_2^2 + \sigma_1^2\big) 
+ \frac{2\beta^2 T_c^2 L L_1}{N}L_G^2 + 2\beta^2 T_c^2 L L_1\mathbb{E} \big[ \big\| \nabla \mathcal{L}(\omega_t) \big\|^2 \big] + \frac{L}{2} \beta b_1^2 \\
&\leq \left(-\frac{3\beta T_c}{4} + 2L\beta^2 T_c^2 L_1\right) \mathbb{E}\big\|\nabla \mathcal{L}(\omega_t)\big\|^2 
+ \left(\beta^3 T_c^3 + 6L\beta^4 T_c^4 L_1 \right) L_7^2 b_3^2b_2^2 \notag \\
&\quad + \bigg(L_7 b_3^2 \beta^3 T_c^3 + 6L\beta^4L_7^2b_3^2T_c^4 L_1 + \frac{2\beta^2 T_c L L_1}{N}\bigg) \sigma_1^2  + \left(\frac{L}{2} + 1\right) \beta b_1^2 
+ \frac{2\beta^2 T_c^2 L L_1}{N}L_G^2
\end{align*}

Therefore

\begin{align*}
\mathbb{E}[\mathcal{L}(\omega_{t+1}) - \mathcal{L}(\omega_t)] \leq -\frac{\beta T_c}{2} \mathbb{E}\big\|\nabla \mathcal{L}(\omega_t)\big\|^2 
+ 2\beta^3 T_c^3 L_7^2 b_3^2 b_2^2 
+ \frac{4\beta^2 T_c L L_1}{N} \sigma_1^2 \notag + \left(\frac{L}{2} + 1\right) \beta b_1^2 
+ \frac{2\beta^2 T_c^2 L L_1}{N} L_G^2
\end{align*}

Hence

\begin{align*}
\mathbb{E} \| \nabla \mathcal{L}(\omega_t)\|^2 &\leq \frac{2}{\beta T_c}\mathbb{E}[\mathcal{L}(\omega_t)-\mathcal{L}(\omega_{t+1})] + 4\beta^2 T_c^2 L_7^2 b_3^2 b_2^2 + \frac{8\beta L L_1}{N}\sigma_1^2 + \frac{L+2}{T_c}b_1^2 + \frac{4\beta T_c L L_1}{N}L_G^2
\end{align*}

And

\begin{align*}
\frac{1}{T} \sum_{t=1}^{T} \mathbb{E} \| \nabla \mathcal{L}(\omega_t)\|^2 & \leq \frac{2(\mathcal{L}(\omega_0)-\mathcal{L}\ast}{\beta T_c} + 4\beta^2 T_c^2 L_7^2 b_3^2 b_2^2 + \frac{8\beta L L_1}{N}\sigma_1^2 + \frac{L+2}{T_c}b_1^2 + \frac{4\beta T_c L L_1}{N}L_G^2
\end{align*}

By setting $\beta=\frac{2\sqrt{N}}{T_c\sqrt{T}}$ we obtain

\begin{align*}
& \frac{1}{T}\sum_{t=1}^{T}\mathbb{E}\big\|\nabla \mathcal{L}(\omega_t)\big\|^2\le \frac{(\mathcal{L}(\omega_0)-\mathcal{L}\ast)}{\sqrt{NT}} + \frac{LL_1(16\sigma_1^2+8T_cL_G^2)}{T_c\sqrt{NT}}+\frac{16NL_7^2b_3^2b_2^2}{T}+\frac{(L+2)}{T_c}b_1^2
\end{align*}
\end{proof}

\section{Appendix - Generalization Proofs}\label[Appendix]{app_generalization-proofs}

Proof for Theorem 2.
\begin{proof}
Using Theorem 4 from \cite{baxter2000model} and the notation used in that paper, we obtain 
$M = \mathcal{O}\left(\frac{1}{N\varepsilon^2} \log\left(\frac{C (\varepsilon, H_c^N)}{\delta }\right)\right)$
where $C (\varepsilon, H_c^N)$ is the covering number for $H_c^N$.
In our case, expected loss is
\begin{align*}
\mathcal{L}(\omega_i, X_0) 
&= \frac{1}{N} \sum_{i=1}^N \mathbb{E}_{p_i} \left[ \mathcal{L}_i(G_i; \omega_i) \right] \\ 
&= \frac{1}{N} \sum_{i=1}^N \mathbb{E}_{p_i} \left[ \mathcal{L}_i\left(G_i; H(S(X_{i}^{(c)}, \theta),\varphi), \omega_{h,i}\right)\right].    
\end{align*}

We assume the following Lipschitz conditions hold:
\begin{align*}
& \quad |\mathcal{L}_i(G_i; \omega_i) - \mathcal{L}_i(G_i; \omega_i^\prime)| \leq \mathcal{L} \| \omega_i - \omega_i^\prime \| \\
& \quad \| \omega_i - \omega_i^\prime \| = \| (\omega_{b,i}, \omega_{h,i}) - (\omega_{b,i}^\prime, \omega_{h,i}^\prime) \| \leq \| \omega_{b,i} - \omega_{b,i}^\prime \| + \| \omega_{h,i} - \omega_{h,i}^\prime \| \\
& \quad \| H(x_i^{(T_s)}, \varphi) - H(x_i^{(T_s)}, \varphi^\prime) \| \leq L_{H} \| \varphi - \varphi^\prime \| \\
& \quad \| H(x_i^{(T_s)}, \varphi) - H(x_i^{(T_s)\prime}, \varphi) \| \leq L_{T_s} \| x_i^{(T_s)} - x_i^{(T_s)\prime} \| \\
& \quad \| S(x_i^{(c)}, \theta) - S(x_i^{(c)}, \theta^\prime) \| \leq L_{S} \| \theta - \theta^\prime \| \\
& \quad \| S(x_i^{(c)}, \theta) - S(x_i^{(c)\prime}, \theta) \| \leq L_{T_o} \| x_i^{(c)} - x_i^{(c)\prime} \|
\end{align*}

From the triangle inequality and the Lipschitz conditions, the distance is given by
\begin{align*}
d &= \underset{p_i}{\mathbb{E}}\left[\frac{1}{N} \sum_{i=1}^N \mathcal{L}(\omega_i, X_0) 
    - \frac{1}{N} \sum_{i=1}^N \mathcal{L}(\omega_i, X_0^\prime) \right] \\
  &\leq \frac{1}{N} \sum_{i=1}^N \underset{p_i}{\mathbb{E}} \left[ 
        \left| \mathcal{L}(h(\varphi, x_i^{(T_s)}), \omega_{h,i}; G_i) 
        - \mathcal{L}(h(\varphi^\prime, x_i^{(T_s)}), \omega_{h,i}; G_i) \right| \right] \\
  &\leq L \cdot \left\| h(\varphi, x_i^{(T_s)}) - h(\varphi^\prime, x_i^{(T_s) \prime}) \right\| 
        + L \cdot \| \omega_{h,i} - \omega_{h,i}' \| \\
  &\leq L \cdot \left( \| h(\varphi, x_i^{(T_s)}) - h(\varphi, x_i^{(T_s)'}) \| 
        + \| h(\varphi, x_i^{(T_s)'}) - h(\varphi', x_i^{(T_s)'}) \| \right) + L \cdot \| \omega_{h,i} - \omega_{h,i}' \| \\
  &\leq L \cdot L_{T_s} \cdot \| x_i^{(T_s)} - x_i^{(T_s)'} \| 
        + L \cdot L_H \cdot \| \varphi - \varphi' \| + L \cdot \| \omega_{h,i} - \omega_{h,i}' \| \\
  &\leq L \cdot L_{T_s} \cdot \| S(X_0; \theta) - S(X_0'; \theta') \| 
        + L \cdot L_H \cdot \| \varphi - \varphi' \| + L \cdot \| \omega_{h,i} - \omega_{h,i}' \| \\
  &\leq L \cdot L_{T_s} \cdot L_{T_o} \cdot \| X_0 - X_0' \| 
        + L \cdot L_{T_s} \cdot  L_S \cdot \| \theta - \theta' \| + L \cdot L_H \cdot \| \varphi - \varphi' \| + L \cdot \| \omega_{h,i} - \omega_{h,i}' \|.
\end{align*}

If each of the parameters $\theta, \varphi, \omega_{h,i}$ and each embedding $X_0$ can be approximated within $\frac{\varepsilon}{4L \left(L_{T_S}(L_S + L_{T_o}) + L_H + 1\right)}$ by a corresponding point, then the resulting set forms an $\varepsilon$-covering under the metric $d$.
From here we get
\begin{align*}
\log\left(C(\varepsilon, H^2)\right) = \mathcal{O}\left(\big( N d_{i} + d_\theta + d_\varphi + d_{\omega_{h,i}}\big)\log \left( \frac{RL(L_{T_S} (L_S + L_{T_o}) + L_H + 1)}{\varepsilon} \right) \right).
\end{align*}

Let $L_s=L_{T_S} \cdot (L_S + L_{T_o})$.
The bound expand to

\begin{align*}
\mathcal{O} \left( \frac{N d_{i} + d_\theta + d_\varphi + d_{\omega_{h,i}}}{N \varepsilon^2} \log \left( \frac{RL(L_{T_S}(L_S + L_{T_o}) + L_H + 1)}{\varepsilon} \right) + \frac{1}{N\varepsilon^2} \log\frac{1}{\delta }\right) 
\end{align*}
\begin{align*}
=\mathcal{O}\left(\big(d_{i}+\frac{d_\theta+d_\varphi+d_{\omega_{h,i}}}{N}\big)\frac{1}{\varepsilon^2}\log{(\frac{RL(L_s+L_H+1)}{\varepsilon})
+\frac{1}{N\varepsilon^2}\log{\frac{1}{\delta}}}\right).
\end{align*}
\end{proof}


\end{document}